\documentclass[letterpaper, 10 pt, conference]{ieeeconf}  

\IEEEoverridecommandlockouts
\makeatletter
\let\NAT@parse\undefined
\makeatother

\usepackage{graphics} 
\usepackage{epsfig} 
\usepackage{times} 
\usepackage{amsmath} 
\usepackage{amssymb}  
\usepackage{mathrsfs}
\usepackage{hyperref}
\usepackage{amsthm}
\usepackage{algorithm}
\usepackage{amsfonts}       
\usepackage{color}
\usepackage{subfigure}

\usepackage{algorithm}
\usepackage{algorithmicx}

\usepackage{booktabs,multirow}
\usepackage{algpseudocode}
\usepackage{threeparttable}
\usepackage{stfloats}

\newtheorem{definition}{\indent
 Definition}

\theoremstyle{remark}
\newtheorem{theorem}{\indent Theorem}
\newtheorem{assumption}{\indent Assumption}
\newtheorem{lemma}{Lemma}

\title{\LARGE \bf
Certificated Actor-Critic: Hierarchical Reinforcement Learning with Control Barrier Functions for Safe Navigation
}

\author{Junjun Xie$^{1\dagger}$, Shuhao Zhao$^{1\dagger}$, Liang  Hu$^{1*}$ and Huijun Gao$^{2*}$
\thanks{$^{\dagger}$ Equal Contribution, $^{*}$ Corresponding Authors.}
\thanks{$^{1}$J. Xie, S. Zhao and L.~Hu are with the Department
of Automation, School of Mechanical Engineering and Automation, Harbin Institute of Technology, Shenzhen,
China.}
\thanks{$^{2}$H. Gao is with Research Institute of Intelligent Control
 and Systems, Harbin Institute of Technology, Harbin, China.}
}

\begin{document}

\maketitle
\thispagestyle{empty}
\pagestyle{empty}

\begin{abstract}
Control Barrier Functions (CBFs) have emerged as a prominent approach to designing safe navigation systems of robots. Despite their popularity, current CBF-based methods exhibit some limitations: optimization-based safe control techniques tend to be either myopic or computationally intensive, and they rely on simplified system models; conversely, the learning-based methods suffer from the lack of quantitative indication in terms of navigation performance and safety. In this paper, we present a new model-free reinforcement learning algorithm called Certificated Actor-Critic (CAC), which introduces a hierarchical reinforcement learning framework and well-defined reward functions derived from CBFs. We carry out theoretical analysis and proof of our algorithm, and propose several improvements in algorithm implementation. Our analysis is validated by two simulation experiments, showing the effectiveness of our proposed CAC algorithm. 
\end{abstract}

\begin{figure*}[t]
\centering
\includegraphics[scale=0.4]{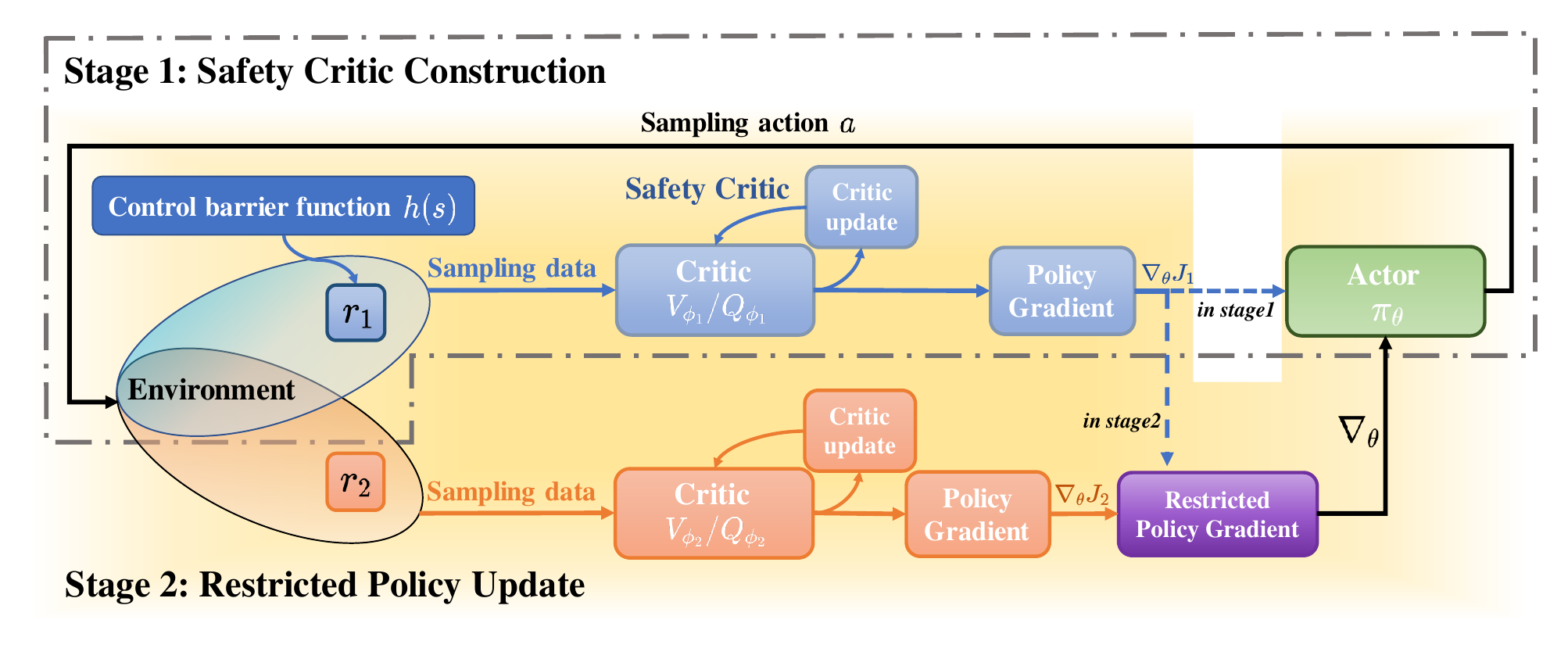}
\vspace{-3mm}
\caption{The framework of certificated actor-critic.}
\vspace{-3mm}
\label{fig:overview}
\end{figure*}

\section{Introduction}\label{sec_intro}
With the increasing deployment of autonomous robots and self-driving vehicles in the real world, the safety of autonomous navigation, e.g., collision avoidance with other users in shared spaces, has emerged as a common societal concern. Recently, the control barrier functions (CBFs) have been employed as a promising technique to address safety issues in robot navigation, either by being incorporated into traditional control and optimization frameworks or by being adapted into model-free learning approaches. 

Building on control theory and optimization, many safe robot navigation approaches have been proposed. A typical paradigm is to divide navigation tasks into path planning and control \cite{gao1,fast2}, with CBFs to ensure safety during path planning \cite{path1,path2,path3}. Another paradigm is to design the navigation controller directly \cite{control1,control2}, in which classic control algorithms are combined with CBF in an optimization framework, such as quadratic programming (QP) \cite{bipedal1,rss1,mobilerobots} and model predictive control (MPC) \cite{MPC-D-CBF}. However, QP-based methods are shortsighted due to pointwise optimization, resulting in locally optimal solutions or even a failed controller \cite{shortview}. On the other hand, MPC mitigates the shortcomings of QP-based method but at the expense of a higher computational burden. In addition, the model-based safe control design usually requires a simplified explicit system model, which limits its application in navigating complex environments.

Reinforcement learning (RL) has long been a popular technique for robot navigation, particularly in scenarios where a model-free approach is required. Recently, researchers have explored various strategies to integrate CBFs with RL to enhance navigation safety \cite{panwei,lidar,cpo}. Some methods impose CBFs as constraints on the learnt policy via a non-invasive approach \cite{aaai_rlcbf,safeRL-cbf,taylor2020learning}. That is, CBFs are used as a ``safety filter'' to adjust the trained policy via QP with CBF constraints. This approach is easy for implementation and has the merits of rigorous safety guarantees, but the adjustment on the origin policy might degrade its performance in an unpredictable approach. Another alternative direction is reward shaping using CBF \cite{ref1,ref2} in which safety is combined together with other reward terms. Since the final reward function is a trade-off between safety and other objectives, the safety requirement cannot be guaranteed, and it lacks quantitative analysis and explainability in terms of safety.

To overcome the limitation in existing methods, we propose a novel model-free reinforcement learning algorithm called \textbf{Certificated Actor-Critic}. Our algorithm features a CBFs-derived reward function that provides a quantitative assessment of navigation safety. Furthermore, we develop a hierarchical RL framework that learns a safe policy first and then refines it at the next stage to achieve fast goal-reaching without compromising safety. 

The main contributions of the paper are threefold:
\begin{enumerate}
    \item We propose a model-free reinforcement learning algorithm certificated Actor-Critic including a CBFs-derived reward function, which can quantitatively estimate the safety of the policies and states;
    \item We design a hierarchical framework that accommodates safety and goal-reaching objectives in robot navigation, and improve its goal-reaching capability yet maintaining safety via novel restricted policy update strategies;
    \item We conduct two experiments with detailed comparative analysis, showing the effectiveness of our proposed algorithm.
\end{enumerate}

\section{Preliminaries}\label{sec_preliminaries}
Consider an infinite-horizon Markov decision process (MDP)  \cite{sutton2018reinforcement} with a discrete-time stochastic system \begin{equation}\label{eq_discretesystem}
    s_{t+1}=\mathcal{F}(s_t,a_t)
\end{equation}
which can be defined concisely as a tuple $\left<\mathcal{S},\mathcal{A},\mathcal{F},R,\gamma\right>$, where $\mathcal{S}$ is a set of states satisfied $s_t\in\mathcal{S}\subseteq\mathbb{R}^n$, $\mathcal{A}$ is a set of actions satisfied $a_t\in \mathcal{A}\subseteq\mathbb{R}^m$, $\mathcal{F}:\mathcal{S}\times\mathcal{A}\to \mathcal{S}$ is a possibly dynamics, $R:\mathcal{S}\times\mathcal{A}\to \mathbb{R}$ is a reward function, and $\gamma\in\left[0,1\right]$ is a discount factor.
\subsection{Reinforcement Learning}
Reinforcement learning aims to learn a policy via interaction between environment and agent with the dynamic model \eqref{eq_discretesystem}. The cumulative reward is defined as $G_t=\sum_{k=0}^\infty\gamma ^k R_{t+k}$. With a policy $\pi$, state value function $V^\pi(s)$ and action-value function $Q^\pi(s,a)$ are defined as
\begin{equation}
\begin{aligned}
    V^\pi(s)&=\mathbb{E}_\pi[G_t|S_t=s],\\
    Q^\pi(s,a)&=\mathbb{E}_\pi[G_t|S_t=s,A_t=a].
    \end{aligned}
\end{equation}
The objective is to find an optimal policy $\pi^*$ to make $V^{\pi^*}(s)\geq V^{\pi}(s),\forall \pi,s\in \mathcal{S}$.

\subsection{Discrete-time Control Barrier Functions}

Consider the discrete-time system \eqref{eq_discretesystem}, we have the following definitions and lemma:
\begin{definition}\label{def_Safeset}
(Safe Set \cite{CBF_2014CDC,CBF_2017TAC,CBF_2019ECC}) The \textbf{safe set} $\mathcal{C}$ is defined as the zero-superlevel set of a smooth function $h:\mathbb{R}^n \to \mathbb{R}$, i.e.,
    \begin{equation}\label{safeset}
        \begin{aligned}
            \mathcal{C} & =\left\{s \in \mathbb{R}^{n}: h(s) \geq 0\right\}, \\
            \partial \mathcal{C} & =\left\{s \in \mathbb{R}^{n}: h(s)=0\right\}, \\
            \operatorname{Int}(\mathcal{C}) & =\left\{s \in \mathbb{R}^{n}: h(s)>0\right\}.
        \end{aligned}
    \end{equation}    
\end{definition}
\begin{definition}\label{def_DT-CBF}
    (Discrete-time Control Barrier Function \cite{DT-CBF_2021ACC}) The function $h:\mathbb{R}^n\rightarrow\mathbb{R}$ is called a \textbf{Discrete-time Control Barrier Function} for the system \eqref{eq_discretesystem} if there exists $\alpha$ such that
    \begin{equation}\label{DT-CBF_constraint}
        \sup _{a_t \in \mathcal{A}}\left[h(\mathcal{F}(s_t,a_t))+(\alpha-1)h(s_t)\right]\geq 0, 0<\alpha<1.
    \end{equation}
\end{definition}

\begin{lemma}\label{thm_forward-invariant} 
(Forward Invariant) The system \eqref{eq_discretesystem} is \textbf{forward invariant} in safe set $\mathcal{C}$, i.e.,
\begin{equation}\label{eq_forward-invariant}
    \forall s_{t}\in \mathcal{C} \Rightarrow s_{t+1}\in \mathcal{C}
\end{equation}
if 
\begin{equation}\label{eq_safe}
    h(s_{t+1})+(\alpha-1)h(s_t)\geq 0,0<\alpha<1.
\end{equation}
\end{lemma}

\begin{proof}
Since $s_t\in\mathcal{C}$, from \autoref{def_Safeset} we have $h(s_t)\geq 0$. If \eqref{eq_safe} holds on, then
$  h(s_{t+1})\geq (1-\alpha)h(s_t)\geq 0$,
which means $s_{t+1}\in\mathcal{C}$.
\end{proof}

\section{Problem Formulation}\label{sec_problem}
In this section, we consider a robot navigation task, and formulate the problem  as follows:

Given the robot dynamics \eqref{eq_discretesystem} with a proper control barrier function $h(s)$, we aim to learn a navigation policy such that:
\begin{enumerate}
    \item (\textit{Safety Guaranteed}) The robot should stay in the safe set defined in \autoref{def_Safeset} with the corresponding control barrier function $h(s)$;
    \item (\textit{Goal reaching}) The robot should reach its navigational goal while keeping in its safe conditions;
    \item (\textit{Safety Certificates}) A quantitative safety certificate should be provided, illustrating the level of safety assured by the navigation policy via numerical values.
\end{enumerate}

\begin{algorithm}[!t]
\caption{Certificated Actor-Critic}
\label{algo:CAC}
\begin{algorithmic}[1]
\State Design the control barrier function $h(s)$ for the system 
 \eqref{eq_discretesystem} with expected decay rate $\alpha_0$
\State \textit{\%Stage 1: Safety Critic Construction}
\State Set reward $r_1$ as \eqref{eq_expreward}, initialize the actor network $\pi_\theta$ with parameters $\theta$ and the critic network $Q_{\phi_1}$ or $V_{\phi_1}$ with parameters $\phi_1$ 
\State Define learning rate $\lambda_\theta,\lambda_{\phi_1}$, the loss function $J_1(\theta)$ for actor and $L_1(\phi_1)$ for critic with $r_1$
\For{each step in training} 
    \State $\theta\leftarrow\theta-\lambda_\theta\nabla_\theta J_1(\theta)$
    \State $\phi_1\leftarrow\phi_1-\lambda_{\phi_1}\nabla_{\phi_1} L_1(\phi_1)$
\EndFor
\State \textit{\%Stage 2: Restricted Policy Update}
\State Set reward $r_2$, initialize the critic network $Q_{\phi_2}$ or $V_{\phi_2}$ with parameters $\phi_2$ for navigation
\State Define learning rate $\lambda_{\phi_2}$, the loss function $J_2(\theta)$ for actor and $L_2(\phi_2)$ for critic with $r_2$ 
\For{each step in training}
    \State 
$\nabla_\theta\leftarrow\eqref{eq_restricted-policy-update}$
    
    \State
$\theta\leftarrow\theta-\lambda_\theta\nabla_\theta $
    \State $\phi_1\leftarrow\phi_1-\lambda_{\phi_1}\nabla_{\phi_1} L_1(\phi_1)$
    \State $\phi_2\leftarrow\phi_2-\lambda_{\phi_2}\nabla_{\phi_2} L_2(\phi_2)$
\EndFor
\State \textbf{Output:} $\theta$, $\phi_1$ and $\phi_2$
\end{algorithmic}
\end{algorithm}
\section{Certificated Actor-Critic}\label{sec_CAC}

In this section, we propose a hierarchical reinforcement learning framework called \textbf{\textit{Certificated Actor-Critic (CAC)}} that extends from the classic actor-critic architecture \cite{sutton2018reinforcement}. We decompose the entire navigation task into two sub-tasks that consider safety and goal-reaching respectively. Accordingly, two separate critics are used to update the actor in two stages, as shown in \autoref{algo:CAC} and \autoref{fig:overview}. 

Different from the common approach that combines multi-objectives into a single objective by weight coefficients, CAC considers safety and goal-reaching in two consecutive stages. In the first stage of \textbf{safety critic construction}, an initial policy is learnt with the reward of safety. Then in the second stage of \textbf{restricted policy update},  the obtained policy from the previous stage is further updated subject to both rewards of safety and that of goal-reaching, which improves goal-reaching performance yet without compromising safety. More significantly, due to the well-designed reward function, the critic trained based on rewards of safety also works as a safety certificate to evaluate safety of the learnt policy.

\subsection{Safety Critic Construction} \label{subsec:safetycritic}
In the first stage, we construct a safety critic and train a safe policy via a well-designed reward  of safety. 

Consider a system with a defined control barrier function $h(s)$ and expected decay rate $\alpha_0$. From \autoref{thm_forward-invariant}, if the action $a_t$ satisfies \eqref{eq_safe} at a particular safe state $s_t$, then the state at the next step $s_{t+1}=f(s_t,a_t)$ is still safe. Otherwise, if \eqref{eq_safe} is violated, we have $h(s_{t+1})+(\alpha_0-1)h(s_t)<0$. Hence, we define 
\begin{equation}\label{eq_reward-safe}
r_1(s_t,a_t)=\delta_h\triangleq \min (h(s_{t+1})+(\alpha_0-1)h(s_t),0)
\end{equation}
and regard $\delta_h$ as a \textbf{safety certificate} for action $a_t$, i.e., if $\delta_h(s_t,a_t)=0$, the system is safe at step $t$.

If every step in an episode satisfies $\delta_h=0$, the system is safe during the whole process. Motivated by this, we call the state value function $V_1^\pi$ and the action value function $Q_1^\pi$ \textbf{safety critics} due to their ability on safety evaluation, formally stated in  \autoref{thm_safety-critic}.

\begin{theorem}\label{thm_safety-critic}
\textbf{(Safety Critic)} Consider the system \eqref{eq_discretesystem} with reward function \eqref{eq_reward-safe}:
\begin{enumerate}
    \item The system \eqref{eq_discretesystem} is safe with policy $\pi$ from the initial state $s_0$ if $s_0\in \mathcal{C}$ and $V_1^\pi(s_0)=0$;
    \item The system \eqref{eq_discretesystem} is safe with policy $\pi$ with the initial state-action pair $(s_0, a_0)$ if $s_0\in \mathcal{C}$ and $Q_1^{\pi}(s_0,a_0)=0$.
\end{enumerate}
\end{theorem}

\begin{proof}
    For 1), if $V_1^\pi(s_0)=0$, we have that the return of any episode from $s_0$ satisfies 
\begin{equation}
\label{profeq}
\sum_{k=0}^\infty\gamma ^k r_1(s_{k},a_{k})=0. 
\end{equation}
From \eqref{eq_reward-safe}, $r_1=\min (h(s_{t+1})+(\alpha_0-1)h(s_t),0)\leq 0$, so the unique solution is $r_1(s_{k},a_{k})=0,\forall k\geq 0$. From \autoref{thm_safety-critic}, $\forall k, s_k\in \mathcal{C}$, i.e., the system is always safe from $s_0$.  

For 2), if $Q_1^{\pi}(s_0,a_0)=0$, then the return of any episode choosing $a_0$ at initial state $s_0$ also satisfies \eqref{profeq}. Similarly, since $r_1\leq 0$, we have $\forall k, s_k\in \mathcal{C}$ from \autoref{thm_safety-critic}, and the system is always safe from $s_0$.
\end{proof}

According to \autoref{thm_safety-critic}, we can use safety critics: 1) to evaluate relative safety between policies; for example, if $V_1^{\pi_1}(s)\geq V_1^{\pi_2}(s)$ at all states, then $\pi_1$ is safer than $\pi_2$; 2) to determine the absolute safety about some states; if $v_\pi(s_0)=0$, then the system is absolutely safe from the initial state $s_0$. It is also worth emphasizing that even the optimal safe policy $\pi^*$ does not necessarily satisfy $V_1^{\pi^*}=0$ because of limited action space $\mathcal{A}$. For example, the robot is too close to avoid the obstacle, for which no irretrievable actions exist. 

After training, we derive an optimal safe policy $\pi^*_\text{safe}$, safety critics $V_1^{\pi^*}$ and $Q_1^{\pi^*}$. What remains is to update the policy for better goal-reaching behaviour, which is left for the next stage of restricted policy update.

\subsection{Restricted Policy Update}\label{subsec_policyupdate}
The core challenge in this stage is how to guarantee the safety of the system when policy is updated for better goal-reaching performance. That is, the values of safety critic do not decrease for all states in the second stage. Here, we make an assumption based on empirical observations that safe states and actions are usually non-unique. For example, in the lane-keeping task, the car could drive safely not only in the middle of the lane but also on the left or right slightly. The assumption is stated as below:
\begin{assumption}\label{asp_parameter-continuty}\textbf{
(Parameter Continuity of Safe Policies)} Given a parameterized safe policy $\pi_\theta$,  $\forall \delta>0$, $\exists\theta'$ subject to $\|\theta'-\theta\|\leq \delta$, then the policy $\pi_{\theta'}$ is safe as well.
\end{assumption}

\autoref{asp_parameter-continuty} implies that safe policies are not isolated but continuous in parametric space, and \autoref{fig:parameter_continuity} shows the relationship between safe policy and safe parameters. It leaves us spaces to update the policy for better goal-reaching performance while keeping safe.

\begin{figure}[!t]
\centering
\includegraphics[scale=0.25]{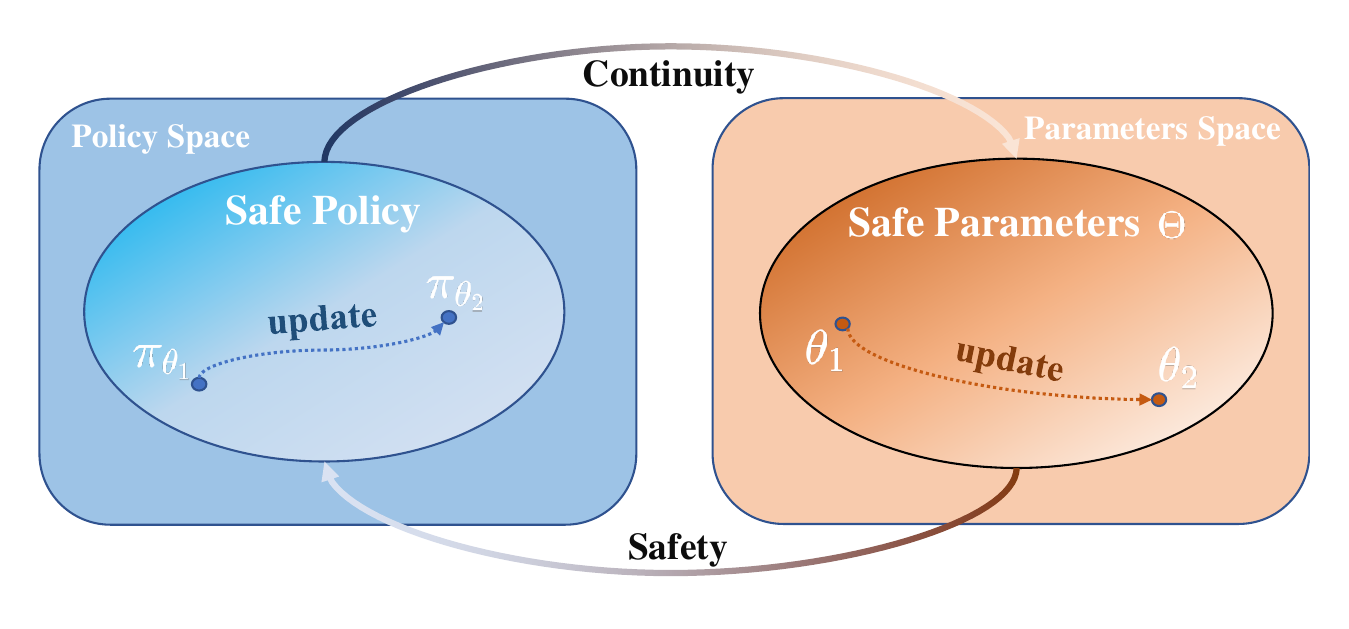}
\vspace{-2mm}
\caption{The relationship between safe policy and its parameters under \autoref{asp_parameter-continuty}. Since the corresponding parameters of similar safe policies  are continuous, parameters can be updated in a small neighbourhood, which guarantees the safety of policy.}
\label{fig:parameter_continuity}
\vspace{-4mm}
\end{figure}

In the second stage, we can set any appropriate reward function $r_2(s_t,a_t)$ for goal-reaching in navigation. For example, set
\begin{equation}\label{eq_reward-stable}
    r_2(s_t,a_t)=-\max(l(s_{t+1})+(\beta_0-1)l(s_t),0)
\end{equation}
where $l(s):\mathbb{R}^n\rightarrow\mathbb{R}$ is a control Lyapunov function (CLF) \cite{CLF_2012CDC,CLF_2014TAC} with an expected decay rate $\beta_0$. In the actor-critic framework, the critic provides the gradient $\nabla_\theta J(\theta)$ for actor (policy) to update. To keep safety critic do not decrease, we should determine the gradient depending on both safety critic and navigation critic. Specifically, we derive the gradient $\nabla_\theta$ using \textbf{restricted policy update} as
\begin{equation}\label{eq_restricted-policy-update}
\begin{aligned}
\nabla_\theta=&\mathop{\arg\max_{e}}e \cdot\nabla_\theta J_2(\theta)\\
\text{s.t.\indent}
&e\cdot \nabla_ \theta J_1(\theta)\geq0\\
&\Vert e\Vert\leq \Vert \nabla_\theta J_2(\theta)\Vert 
\end{aligned}
\end{equation}
where $J_1(\theta), J_2(\theta)$ are actor loss functions with reward $r_1,r_2$ respectively, and $\nabla$ is gradient operator. Under \autoref{asp_parameter-continuty}, the policy $\pi_\text{safe}^*$ will update and converge gradually to the final optimal policy $\pi^*$ along $\nabla_\theta$.

\vspace{-3mm}
\subsection{Practical Improvements}
So far, we have introduced the overall framework of our algorithm, which can be transplanted into any existing actor-critic architecture. However, for training in deep reinforcement learning where all actors and critics are represented as neural networks, we need to make a few improvements.

\subsubsection{Policy improvement}
Consider the parameterized actor network $\pi_\theta$, a widely used policy gradient method is to maximize the function $J(\theta)=\mathbb{E}[V^{\pi_\theta}(s)]$. Since the target is an expectation of state values, there may be cases that some values rise and others decline. On safety, it means some states become safer but others more dangerous. Hence, we improve the policy in the way soft actor-critic \cite{SAC_2018ICML,SAC_2018arxiv} does, that is using Kullback-Leibler divergence as $J(\theta)$ leads to
\begin{equation}\label{eq_softpolicyimproment}
    \pi_{\text {new }}=\arg \min _{\pi^{\prime} \in \Pi} D_{\mathrm{KL}}\left(\pi^{\prime}\left(\cdot \mid \mathbf{s}_{t}\right) \| \frac{\exp \left(Q^{\pi_{\text {old }}}\left(\mathbf{s}_{t}, \cdot\right)\right)}{Z^{\pi_{\text {old }}}\left(\mathbf{s}_{t}\right)}\right)
\end{equation}
and a significant conclusion is $Q^{\pi_\text{new}} (s,a) \geq Q^{\pi_\text{old}}(s,a)$ \cite{SAC_2018ICML,SAC_2018arxiv}. It shows that all states become safer after policy update.
\subsubsection{Exponential reward normalization}
We define $r_1$ in \eqref{eq_reward-safe}, and the best reward at each step is $0$. However, in finite episode, it leads to early termination, since the longer the episode is, the more negative the cumulative reward is. So we adjust reward function to
\begin{equation}\label{eq_expreward}
    r_1(s_t,a_t)=\exp(\min (h(s_{t+1})+(\alpha_0-1)h(s_t),0))
\end{equation}
and obviously $r_1\in (0,1]$. Besides, if there exists maximum length $k_\text{max}$ of an episode , then $V^\pi_1\leq \frac{1-\gamma^{k_\text{max}}}{1-\gamma}$, which means the critic network still can work as the safety critic with a bias $\frac{1-\gamma^{k_\text{max}}}{1-\gamma}$, i.e., if $V^\pi_1(s_0)= \frac{1-\gamma^{k_\text{max}}}{1-\gamma}$, then the system is safe from the initial state $s_0$.

\subsubsection{Gradient enhancement}
A restricted gradient is derived in \eqref{eq_restricted-policy-update} to guarantee safety performance during 
the second stage. Although $\nabla_\theta\cdot \nabla_\theta J_1(\theta)\geq0$ works theoretically, there exist two reasons for possible failure in algorithm implementation. One is there always needs a step length to update the parameters, and local gradient does not guarantee global convergence. Another is the true gradient $\nabla_\theta J_1(\theta)$ is unknown, and is replaced by the estimation $\widetilde{\nabla_\theta} J_1(\theta)$ derived from data. Hence, it is essential to enhance the constraint as $\nabla_\theta\cdot \widetilde{\nabla_\theta} J_1(\theta)\geq \delta$ or $\cos\{\nabla_\theta ,\widetilde{\nabla_\theta} J_1(\theta)\}\geq \delta,\delta>0$ where $\cos\{\cdot,\cdot\}$ represents the cosine of angle between two vector.

\section{Experiments}\label{sec:experiments}
In this section, we validate our CAC algorithm on two simulation experiments: a classic control task CartPole using Gymnasium \cite{gymnasium}, and an autonomous underwater vehicle (AUV)  navigation task using the simulator HoloOcean \cite{HoloOcean}.

\begin{figure}[!t]
\vspace{-4mm}
\begin{center}
\subfigure[Key frames with $\pi^*_\text{safe}$]{\includegraphics[width=0.45\linewidth]{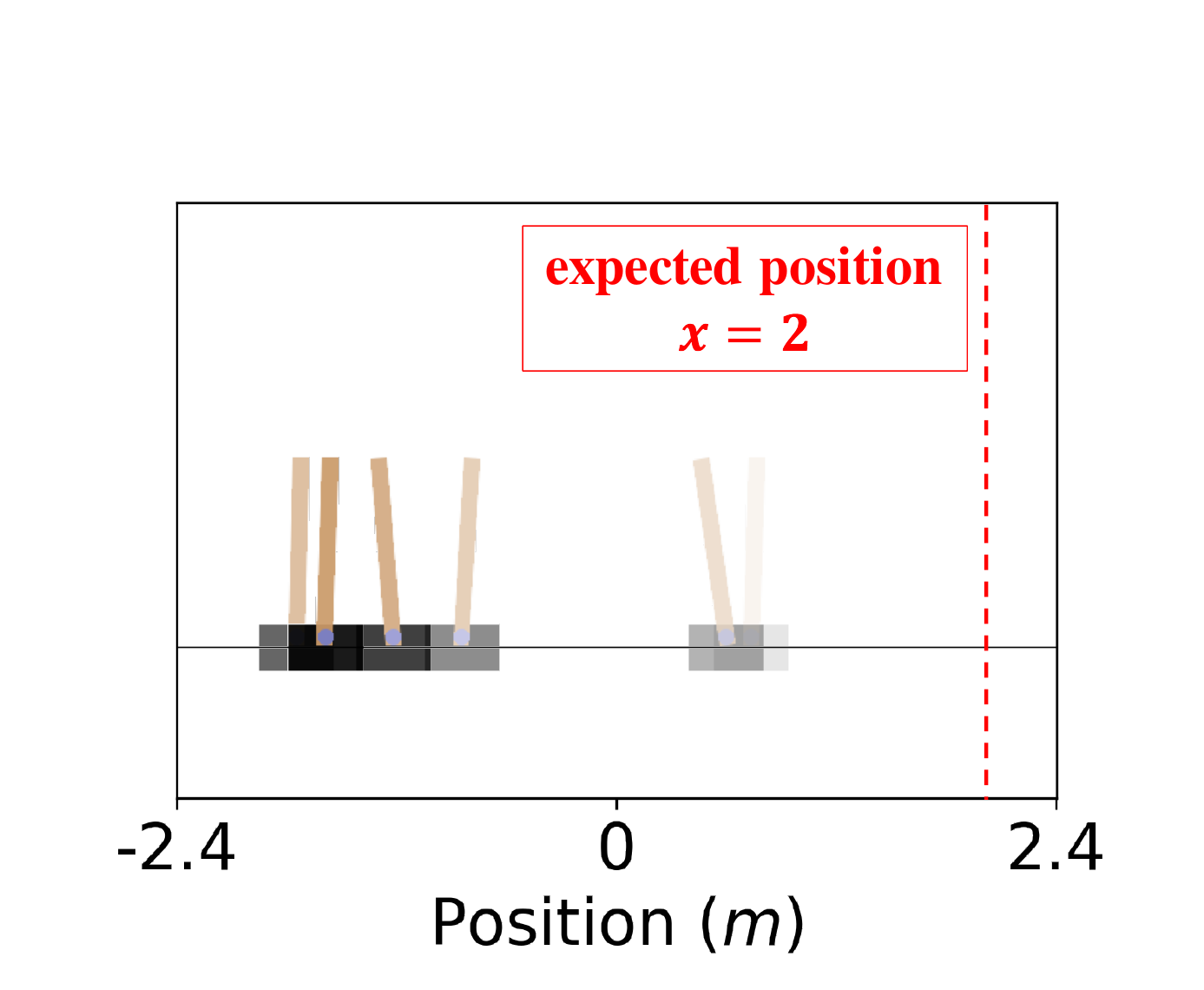}
\label{fig:cartpole_frames_step1}
}
\subfigure[Key frames with $\pi^*$]{\includegraphics[width=0.45\linewidth]{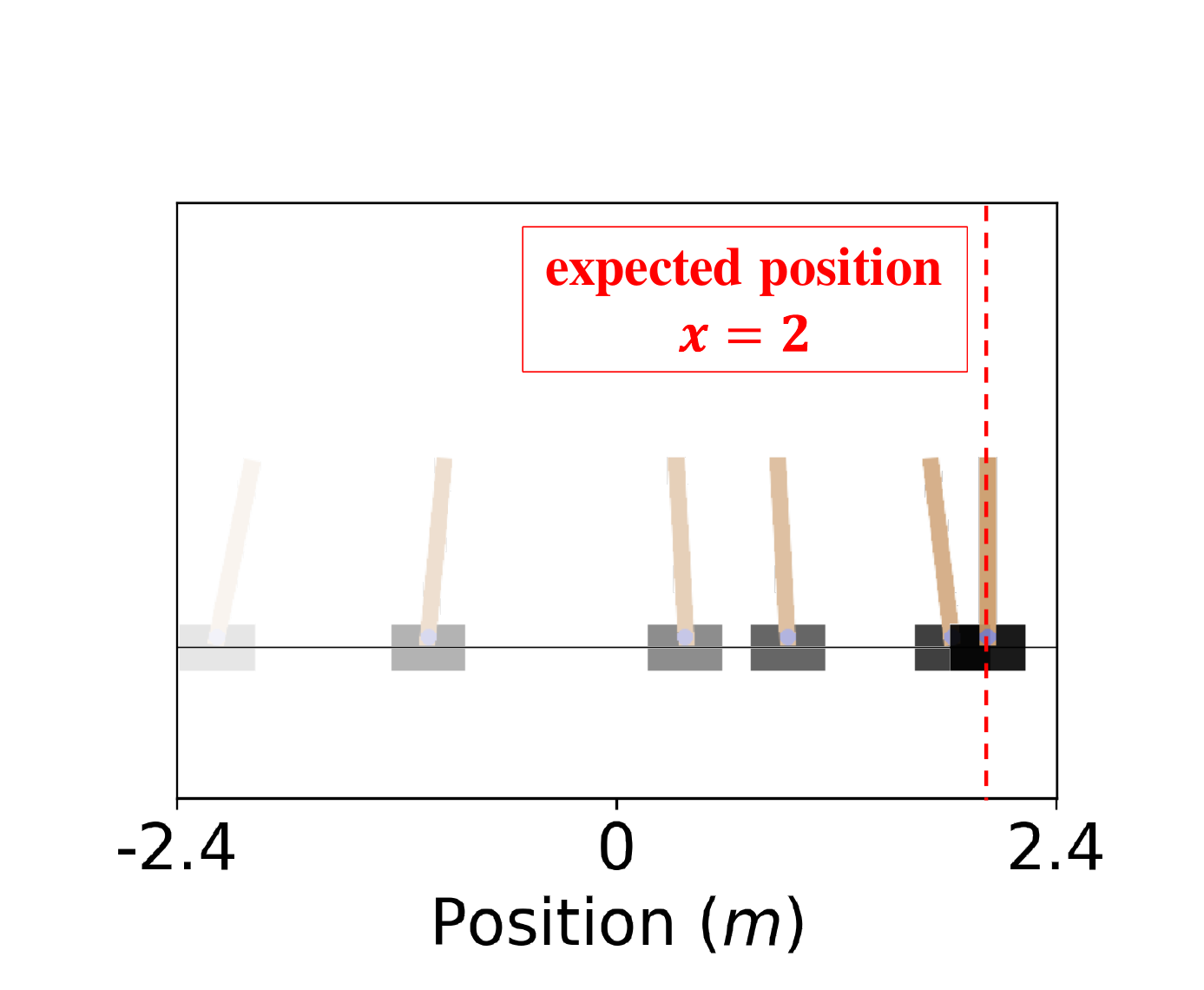}
\label{fig:cartpole_frames_step2}
}
\end{center}
\vspace{-2mm}
\caption{Frames in an episode with $\pi_\text{safe}^*$ and $\pi^*$.}
\vspace{-0.4cm}

\label{fig:frames_cartpole}
\end{figure}

\begin{figure}[!t]
\centering
\subfigure[Positions with $\pi_\text{safe}^*$]{\includegraphics[width=0.5\linewidth]{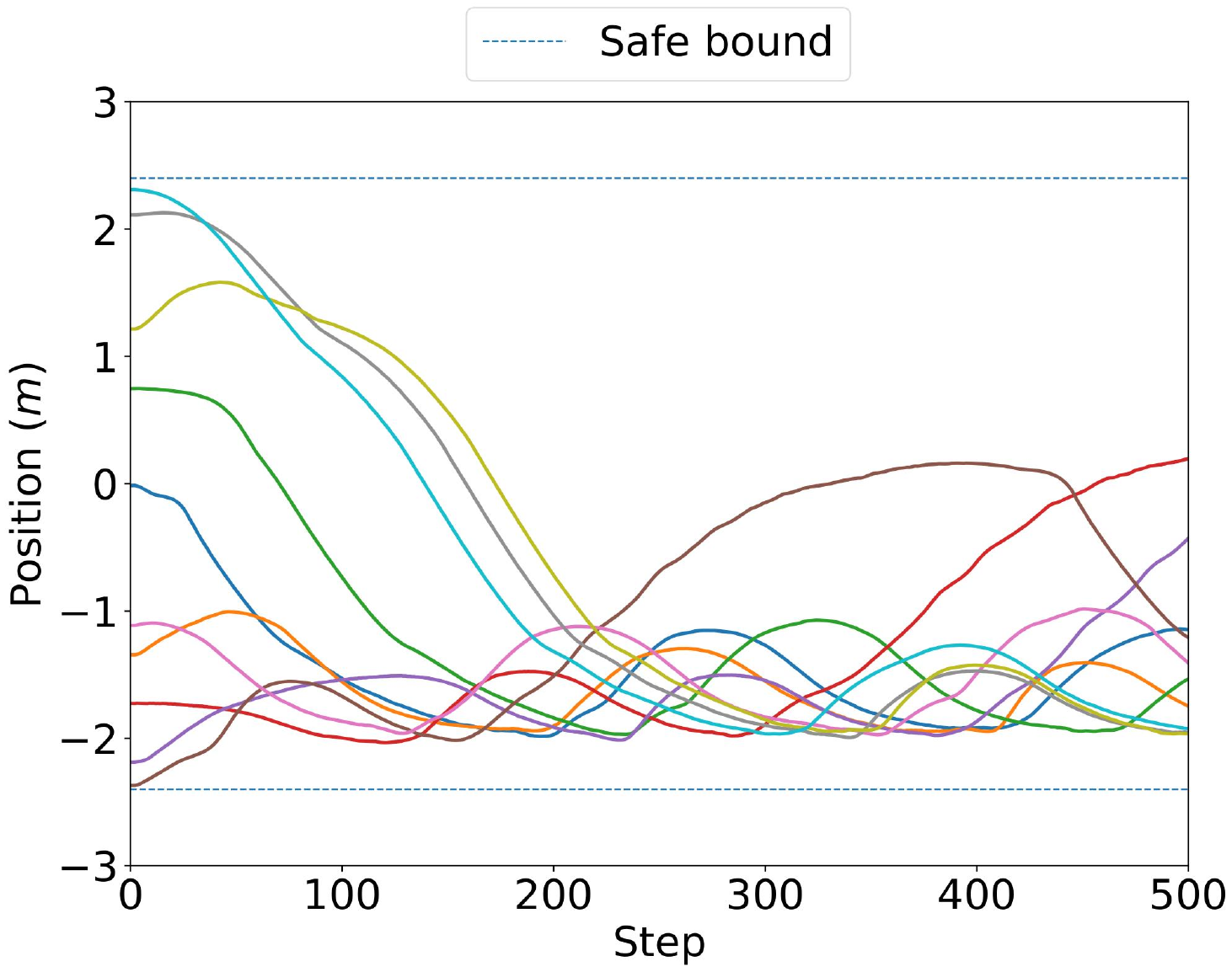}
\label{fig:cartpole_safeposition}
}
\hspace{-5mm}
\subfigure[Angles with $\pi_\text{safe}^*$]{\includegraphics[width=0.5\linewidth]{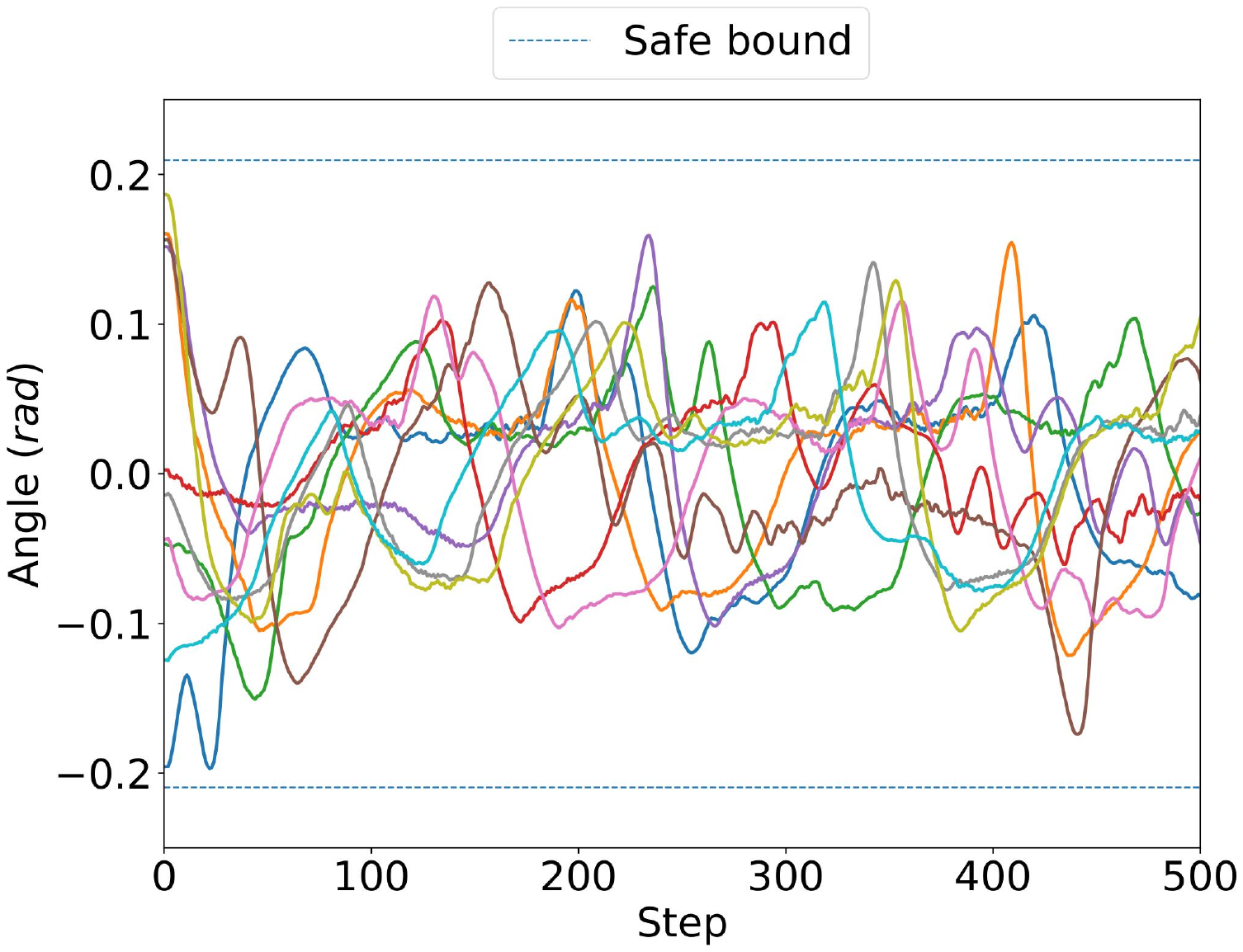}
\label{fig:cartpole_safeangle}
}
\vspace{-3mm}

\subfigure[Positions with $\pi^*$]{\includegraphics[width=0.5\linewidth]{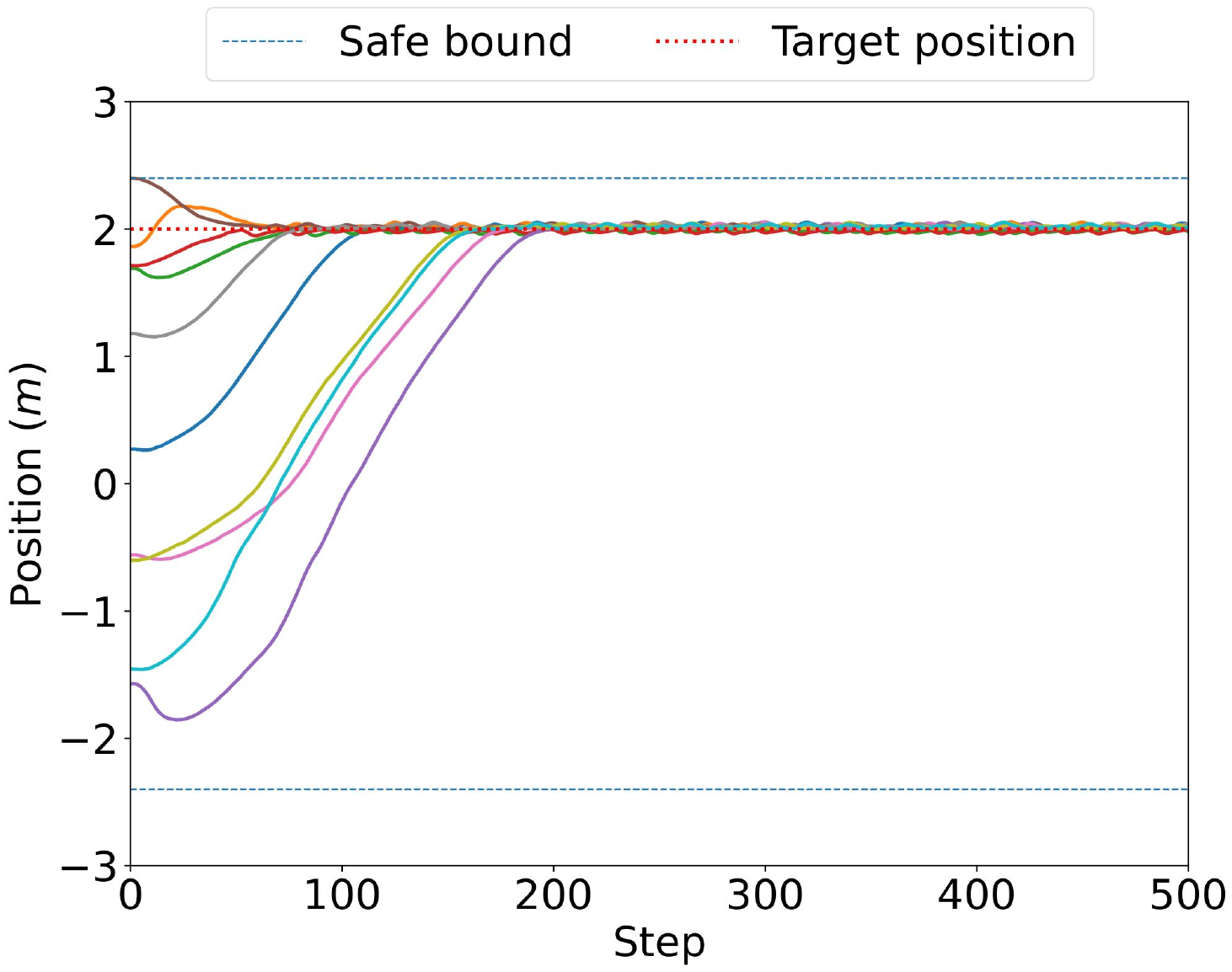}
\label{fig:cartpole_finalposition}
}
\hspace{-5mm}
\subfigure[Angles with $\pi^*$]{\includegraphics[width=0.5\linewidth]{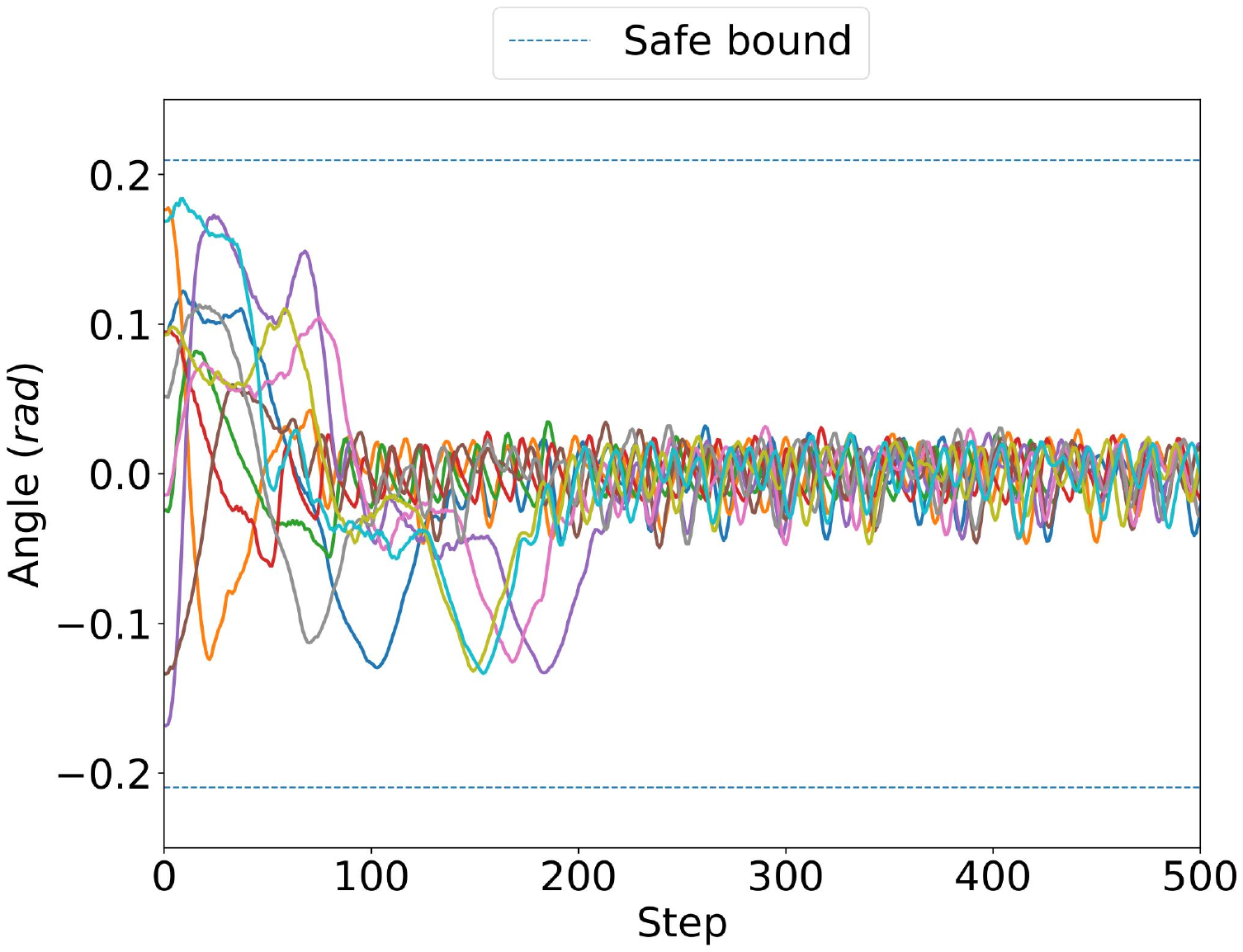}
\label{fig:cartpole_finalangle}
}
\vspace{-3mm}
 \caption{Positions and angles of CartPole with $\pi_\text{safe}^*$ and $\pi^*$ in 10 test episodes. Both policies guarantee the state in the safe range, and the final policy $\pi^*$ drives the CartPole to the target position in safe conditions.}
 \vspace{-4mm}
\end{figure}
\subsection{Continuous CartPole}
\subsubsection{Basic Experiment}\label{sub2sec:cartpole_basic}
CartPole is a typical benchmark for reinforcement learning and control. The task is to balance the pole as long as possible by applying the force on the cart. The state is defined as $s=\left[ \begin{array}{cccc}
    x & v&\theta&\omega 
\end{array}
\right] ^T$, where $x,v$ are the position, velocity of the cart, and $\theta,\omega$ are the angle, angular velocity of the pole respectively. We set the environment with the continuous action space $a\in\left[-1,1\right]$ and the default state space. The navigation  requirements are
\begin{equation*}
\begin{aligned}
    \text{Safety (allowable states): }&-2.4\leq x \leq 2.4
    \\&-12^{\circ}\leq\theta\leq 12^{\circ}\\
    \text{Navigation destination (desired states): }& x=2
    \end{aligned}
\end{equation*}
The initial state is sampled randomly from allowable states.

For safety guaranteed, we design two control barrier functions $h_1(s)=(2.4^2-x^2)/2.4^2$ and $h_2(s)=(12^2-\theta^2)/12^2$. From \eqref{eq_reward-safe} and \eqref{eq_expreward}, the reward of the first stage is $r_1=[\exp(\delta_{h_1})+\exp(\delta_{h_2})]/2$.

\begin{figure}[!t]
\begin{center}
\subfigure[]{\includegraphics[width=0.33\linewidth]{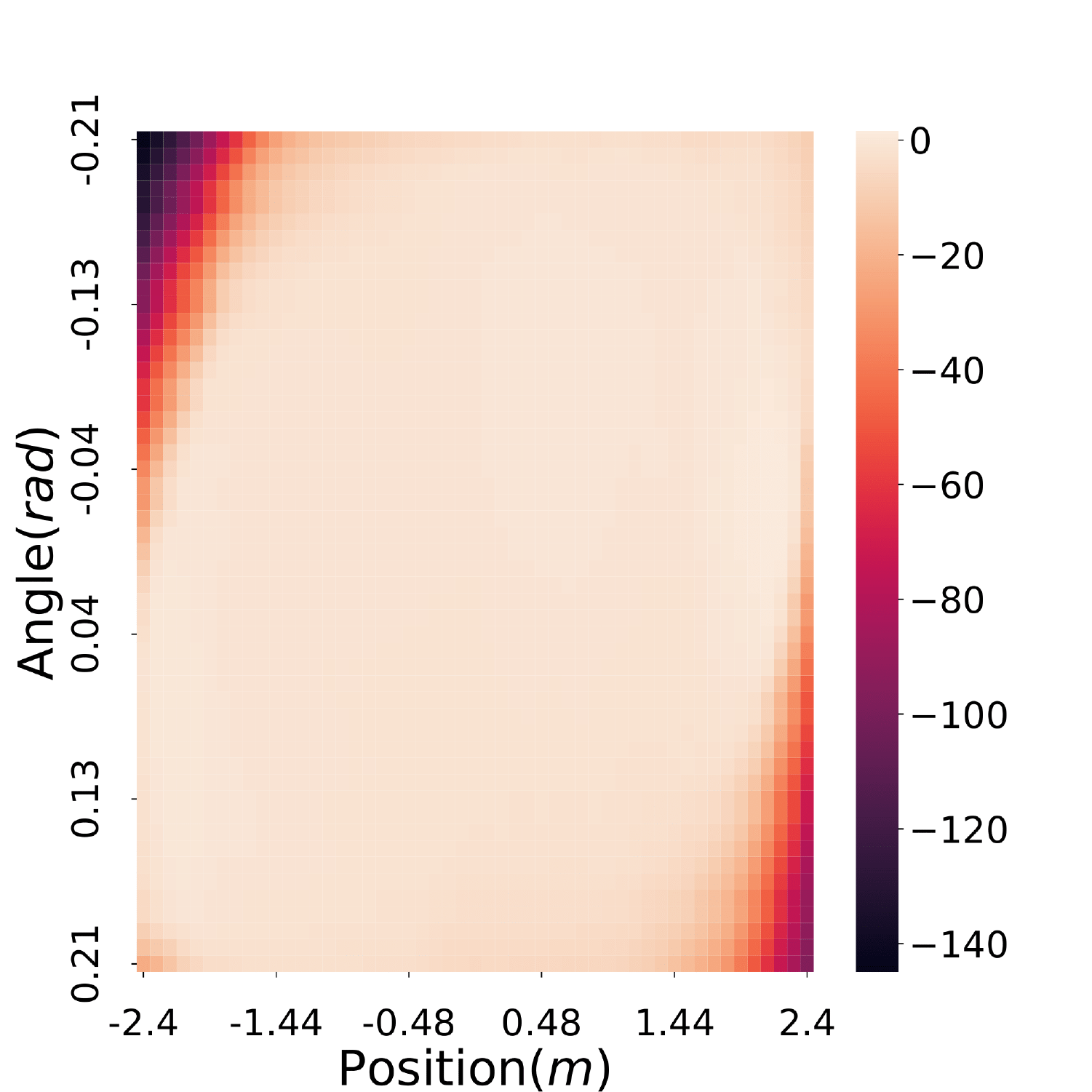}
}
\hspace{-4mm}
\subfigure[]{\includegraphics[width=0.33\linewidth]{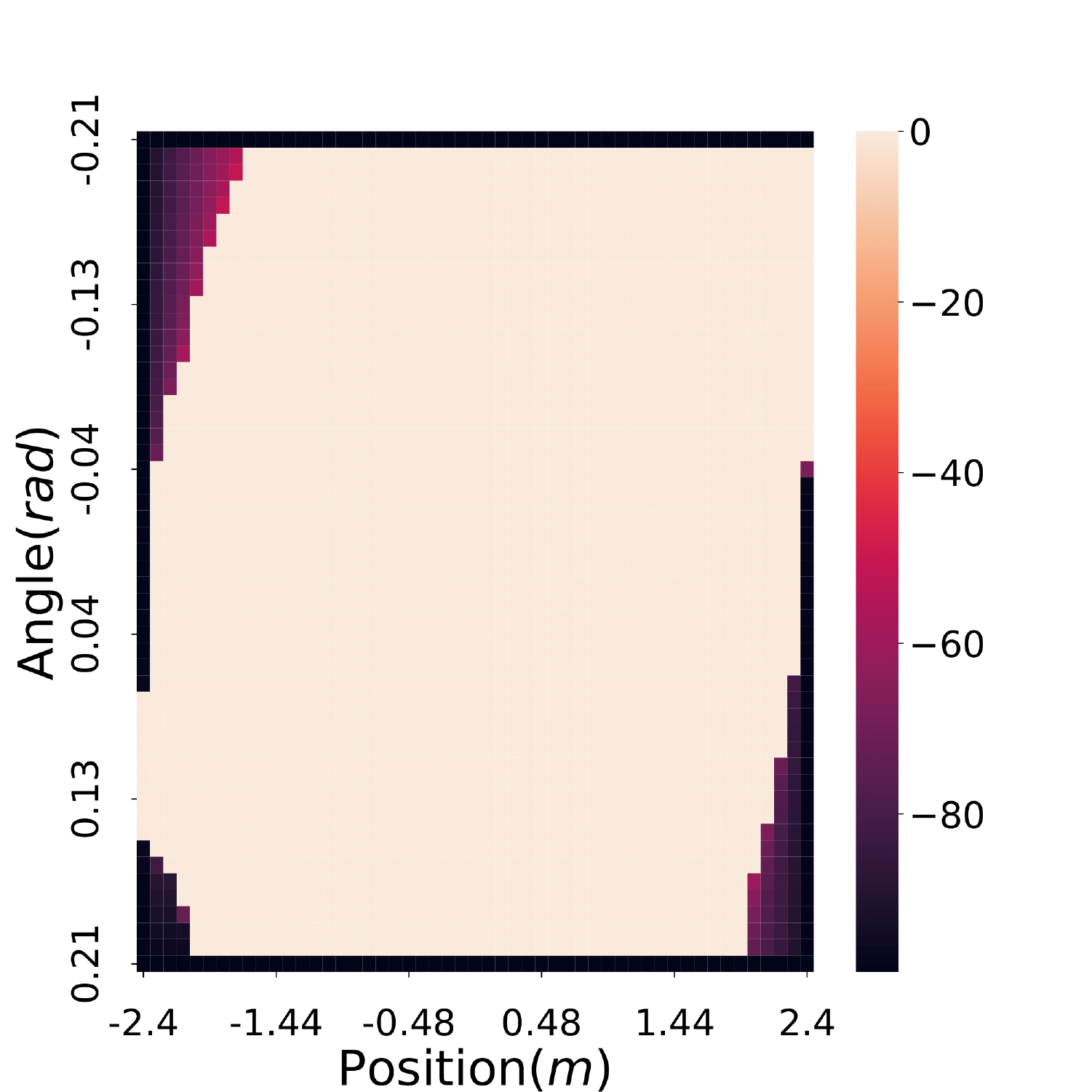}
}
\hspace{-4mm}
\subfigure[]{\includegraphics[width=0.33\linewidth]{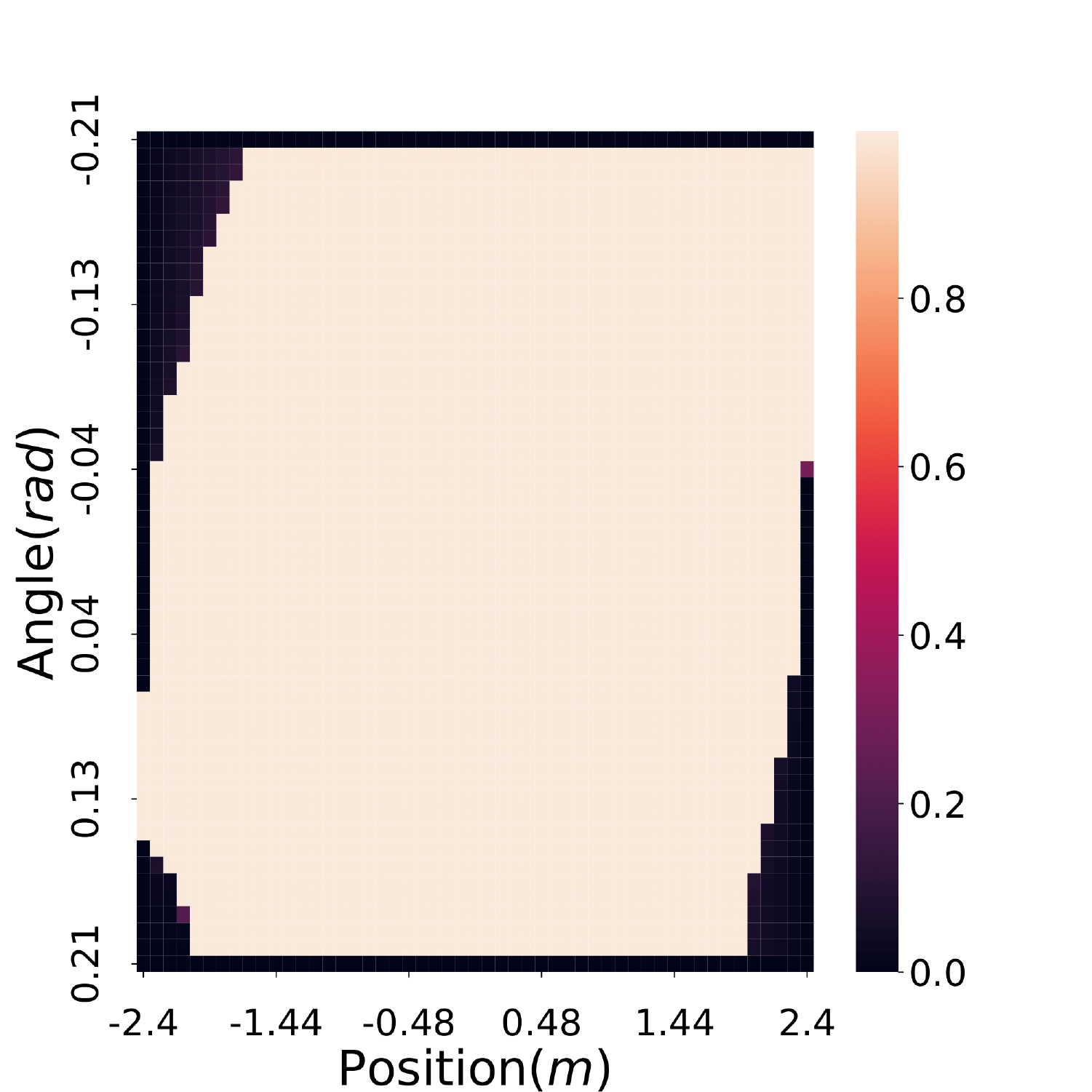}
}
\end{center}
\vspace{-0.2cm}
\caption{Heatmaps of (a) safety critic value, (b) average sampling return and (c) safe rate. Three heatmaps are similar and consistent, which validates that the safety critic is a good safety certificate.}
\label{fig:heatmaps_cartpole}
\end{figure}

\begin{figure}[!t]
\centering
\includegraphics[scale=0.5]{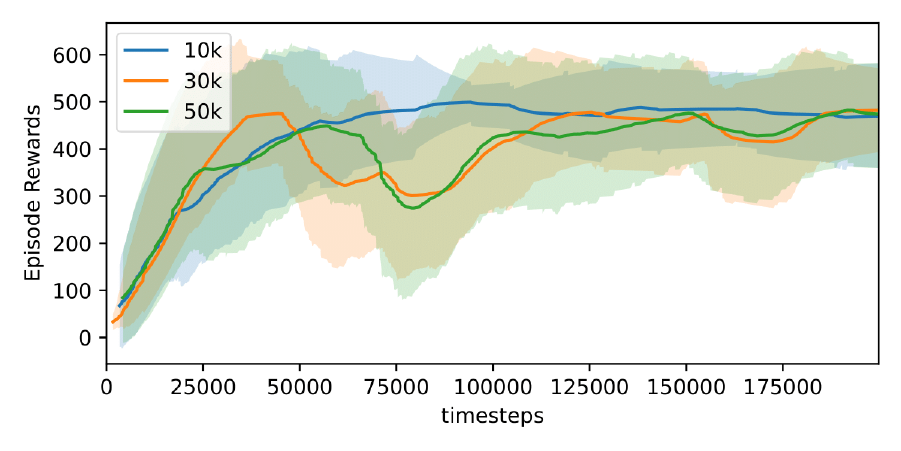}
\caption{Episode rewards on safety (maximum is 500) of the CAC models that are trained with different lengths at the first stage. All models perform almost the same in safety after 200k steps of training at stage 2.}
\vspace{-3mm}
\label{fig:cartpole_training}
\end{figure}

\subsubsection{Verification of Safety Critic $V$}\label{sub2sec:cartpole_verification}

\begin{table}[t]
\caption{Safe Rate in 100 Episodes}
\label{tab:cartpole_saferate}
\begin{center}
\vspace{-4mm}
\begin{tabular}{lcccc}
\toprule
\centering
Length of stage 1 &Stage 1&After 100k steps of stage 2&Stage 2\\
\midrule
10k steps&63\%&99\%&100\%\\
30k steps&95\%&100\%&100\%\\
50k steps&100\%&100\%&100\%\\
\bottomrule
\end{tabular}
\vspace{-0.5cm}

\end{center}
\end{table}

As mentioned in \autoref{thm_safety-critic}, thanks to our well-defined reward, the critic network of the first stage represents the safety of the policy. To verify the conclusion, we compare the safety critic value $V_1(s)$ with the average sampling return and safe episodes rate (an episode is safe if its length is $k_\text{max}=500$) from 10 episodes with all states using $\pi_\text{safe}^*$. The results are depicted in \autoref{fig:heatmaps_cartpole}. Obviously, the three heatmaps are consistent, which verifies the conclusion in \autoref{thm_safety-critic}.

\begin{figure*}[b]
\centering
\subfigure{\includegraphics[width=0.2\linewidth]{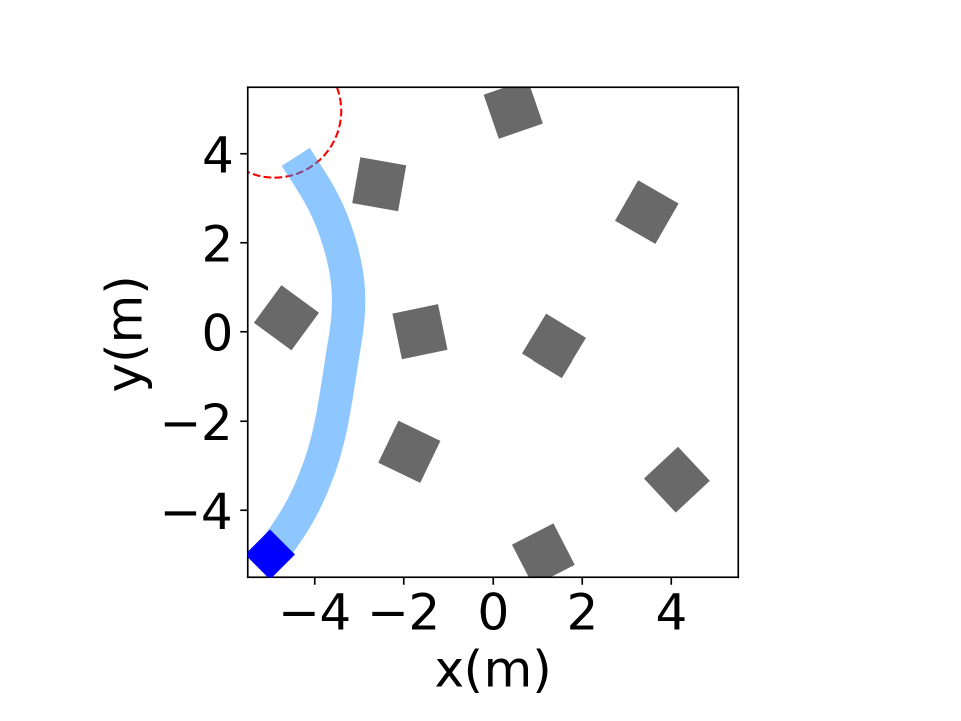}

}
\hspace{-10mm}
\subfigure{\includegraphics[width=0.2\linewidth]{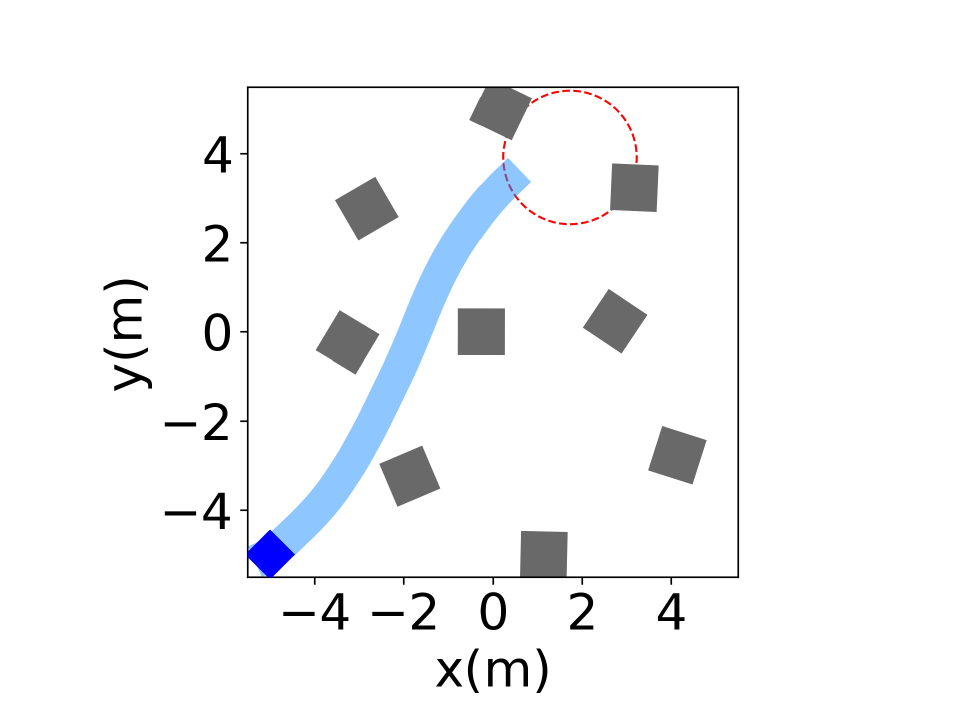}

}
\hspace{-10mm}
\subfigure{\includegraphics[width=0.2\linewidth]{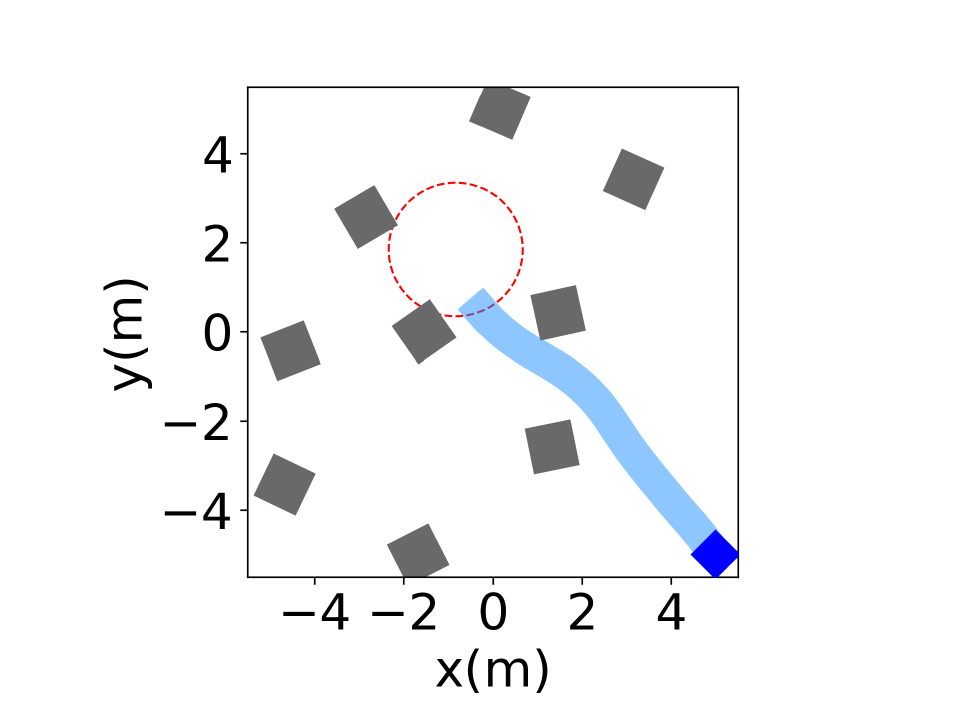}

}
\hspace{-10mm}
\subfigure{\includegraphics[width=0.2\linewidth]{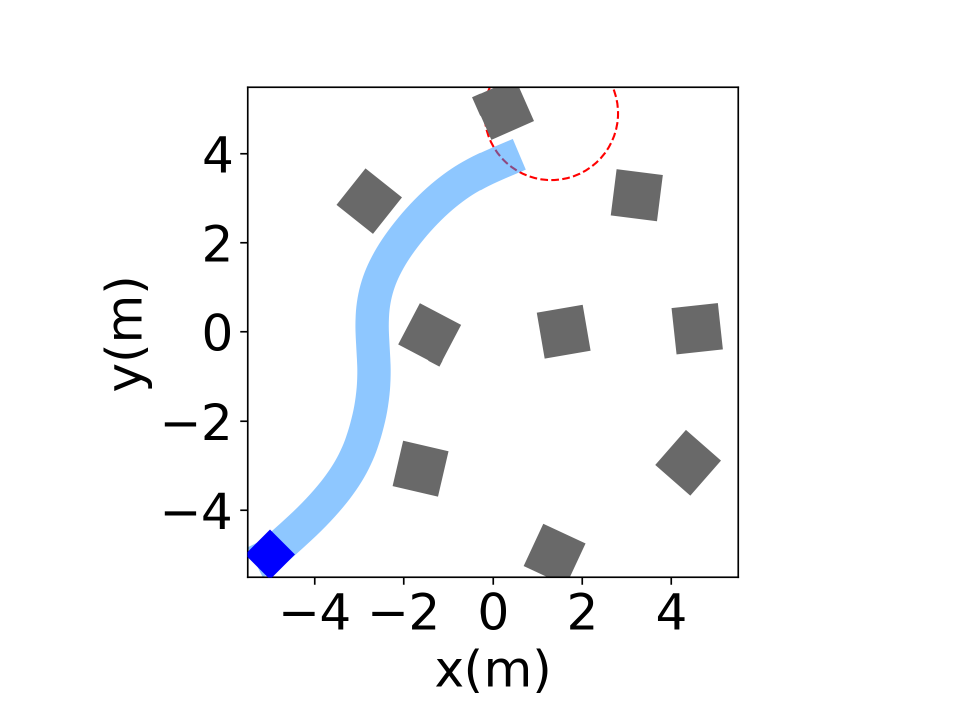}

}
\hspace{-10mm}
\subfigure{\includegraphics[width=0.2\linewidth]{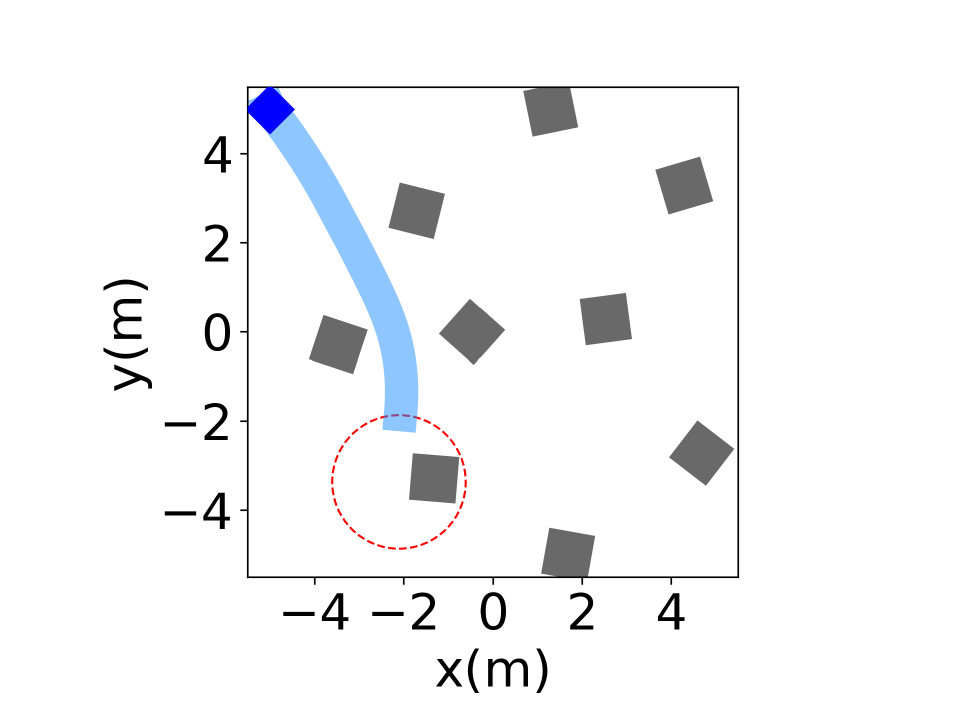}

}

\vspace{-4mm}
\hspace{0.3mm}
\subfigure{\includegraphics[width=0.2\linewidth]{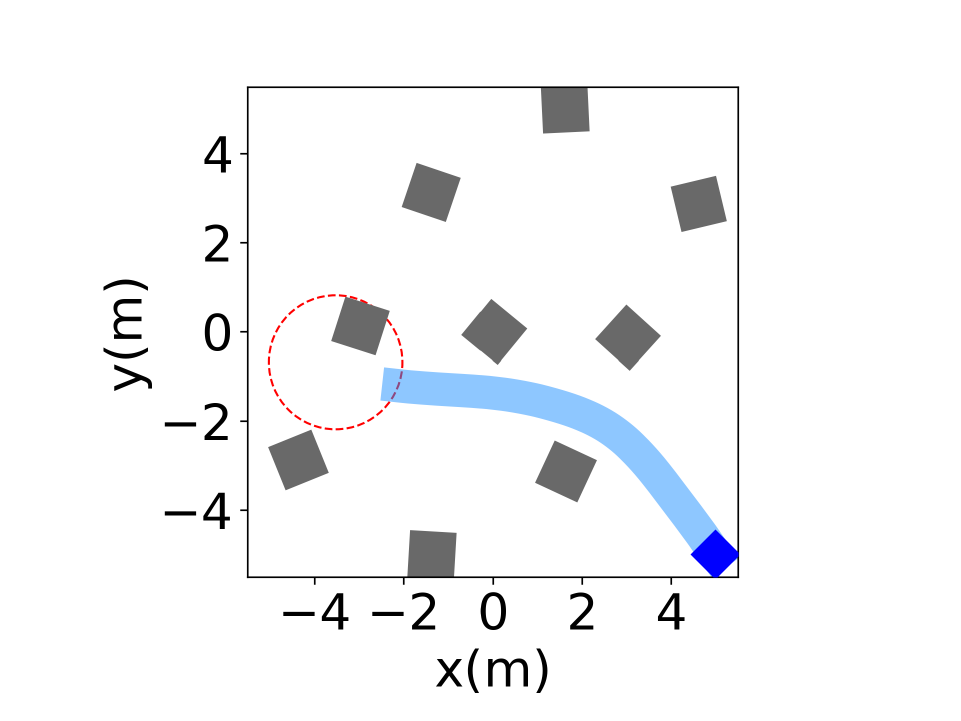}
}
\hspace{-10mm}
\subfigure{\includegraphics[width=0.2\linewidth]{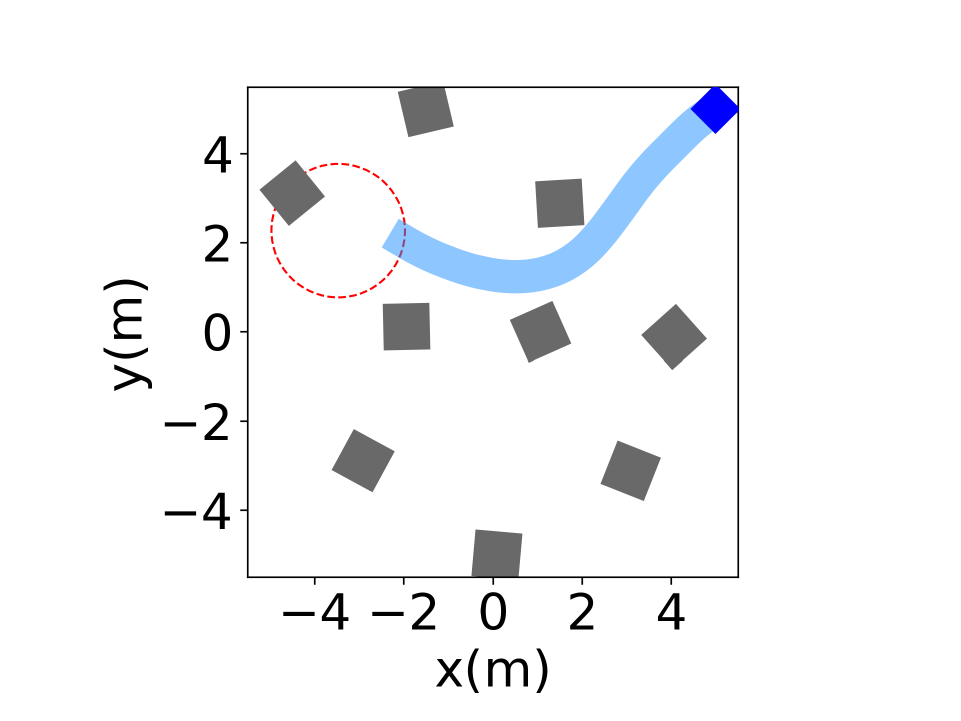}
}
\hspace{-10mm}
\subfigure{\includegraphics[width=0.2\linewidth]{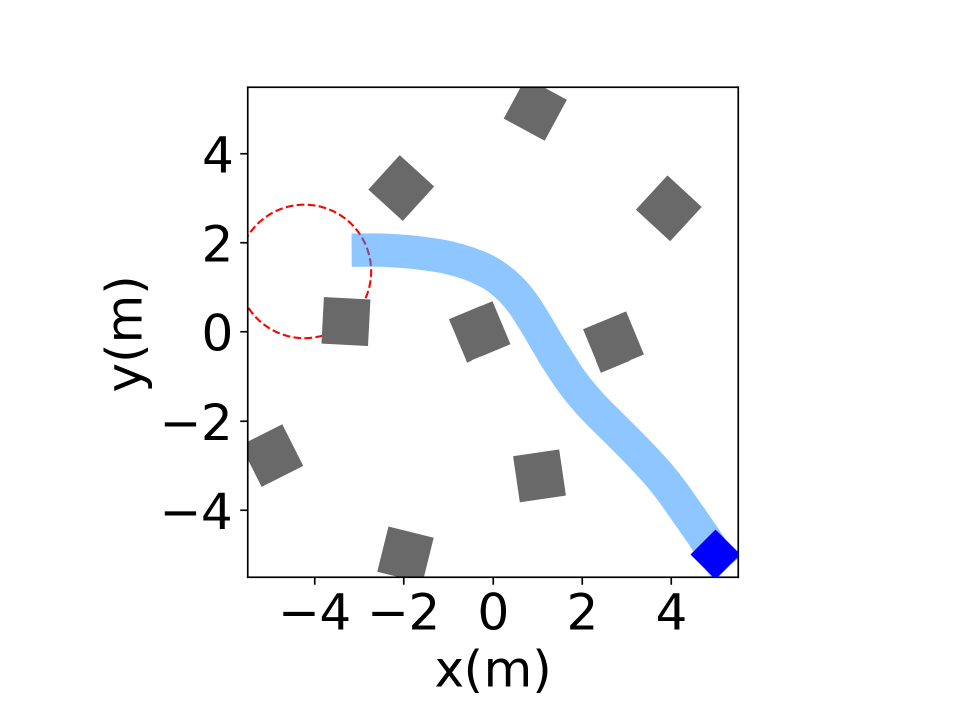}

}
\hspace{-10mm}
\subfigure{\includegraphics[width=0.2\linewidth]{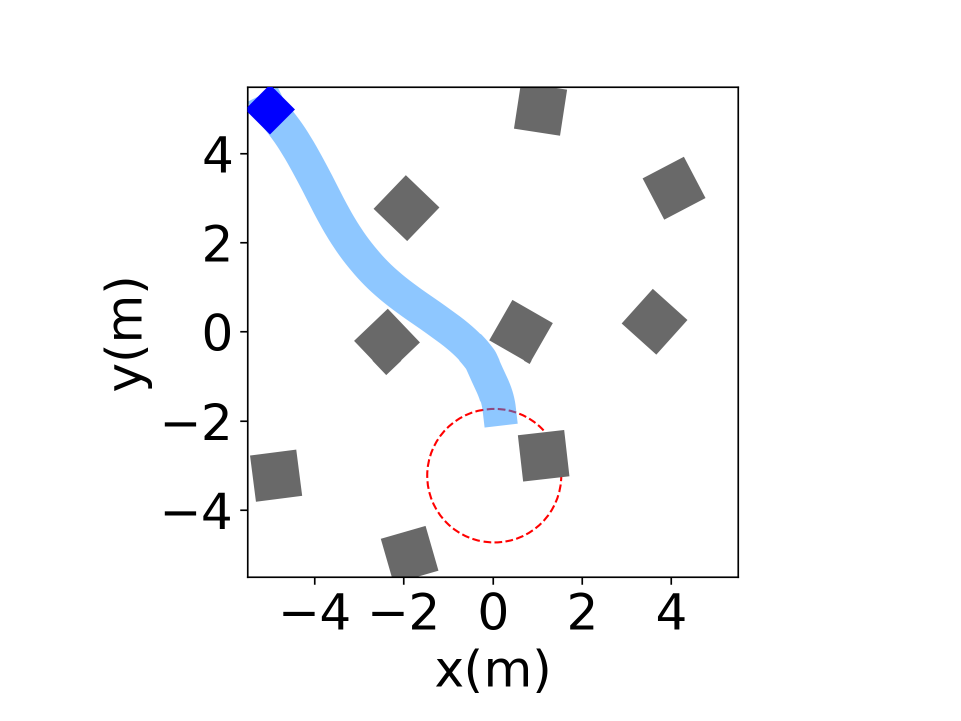}

}
\hspace{-10mm}
\subfigure{\includegraphics[width=0.2\linewidth]{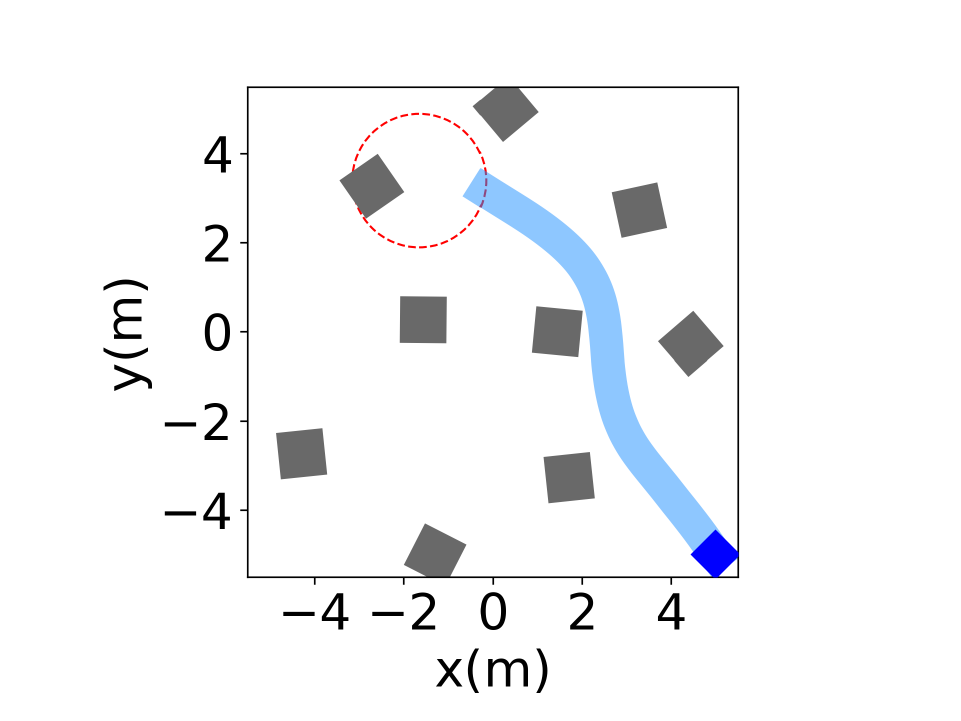}

}
 \caption{Results of test trajectories (light blue). There are 9 random cubic obstacles (grey) in each case. The initial position (dark blue) is set as one of four corners of the pool, and target position (red circle) is generated in the diagonal half of the initial position randomly.}
 \label{fig:traj}
\end{figure*}
Furthermore, we also study the effect of the first stage performance on the final results. Specifically, we respectively train 10k, 30k and 50k steps in the first stage, and 200k steps in the second. The policy networks with 10k and 30k steps do not fully converge in the first stage, but all three final models after the second stage work well. It shows that even if the first stage is not fully trained, the CAC algorithm is still effective with sufficient training steps in the second stage. Moreover, we test the models in the process, which are trained 100k steps in the second stage, and there are significant improvements or maintenance of a high level in safety for all models. shown as \autoref{fig:cartpole_training} and \autoref{tab:cartpole_saferate} .
\begin{figure}[H]
    \centering
    \vspace{2mm}
    \subfigure[Average navigation step reward]{
		\includegraphics[scale=0.38]{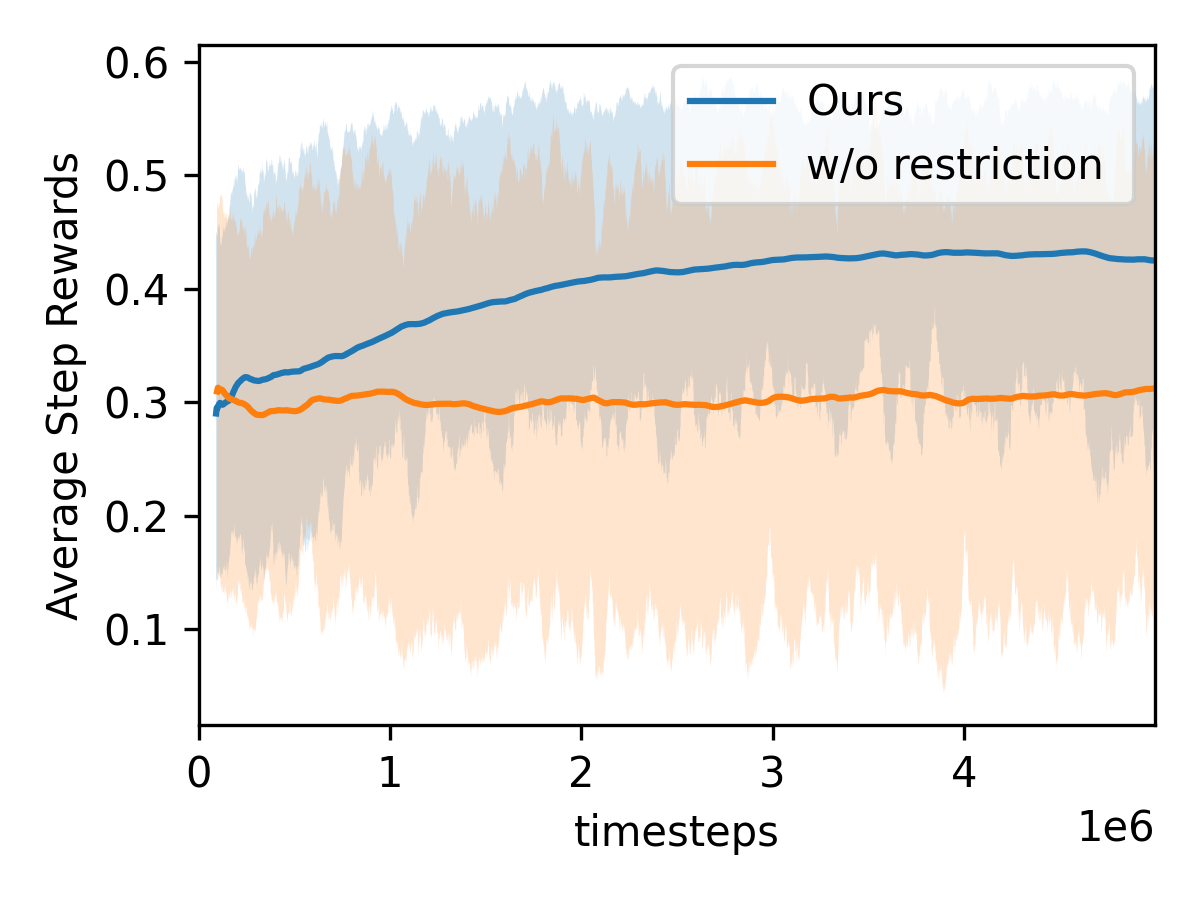}
	}
    \subfigure[Length of episodes]{
\includegraphics[scale=0.38]{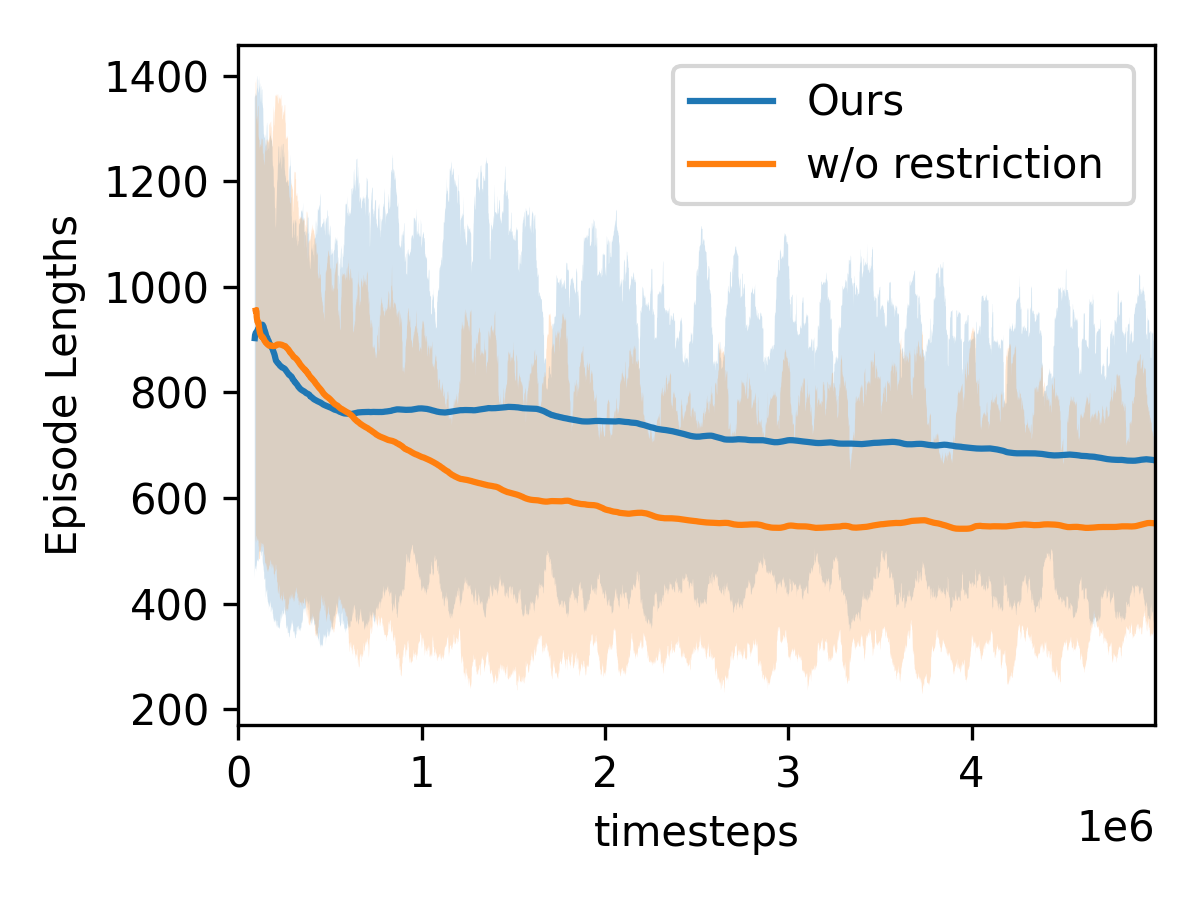}
	}
 \caption{Comparison of average navigation reward per step and length of training episodes in the underwater environment for the CAC models trained with and without policy restriction \eqref{eq_restricted-policy-update}. }
 \label{fig:ab_curve}
\end{figure}
\vspace{-4mm}
\subsection{Autonomous Underwater Vehicle}
To examine the safe navigation capability in complex environments with dense obstacles, we consider an Autonomous Underwater Vehicle (AUV)  navigation task using the HoloOcean simulator \cite{HoloOcean}.

In simulation, the HoveringAUV is kept at a fixed depth. The forward velocity is set to 0.9m/s and the yaw angular velocity range is (-1,1) rad/s. The velocities are followed via a simple PD controller. The AUV is equipped with a RangeFinder sensor with  $120^{\circ}$ FOV, whose measurement is a $\mathbb{R}^9$ vector $\left [ d_1,\cdots,d_9 \right ] $ representing the distances with $15^{\circ}$ interval (the maximum distance is 10 meters). The other observations are the yaw angular velocity $\omega_\gamma$ from the IMU sensor, horizontal velocities $\left [ v_x,v_y \right ]$, the location of AUV $\left [x,y\right ]$, the rotation yaw angle $\gamma$ of AUV in the global frame and the goal location $\left [x_{tar},y_{tar}\right ]$. All observation vectors are concatenated into a $\mathbb{R}^{17}$ vector after normalization.

\begin{figure*}[t]
    \centering
\hspace{-2mm}
\subfigure[]{
\includegraphics[scale=0.22]{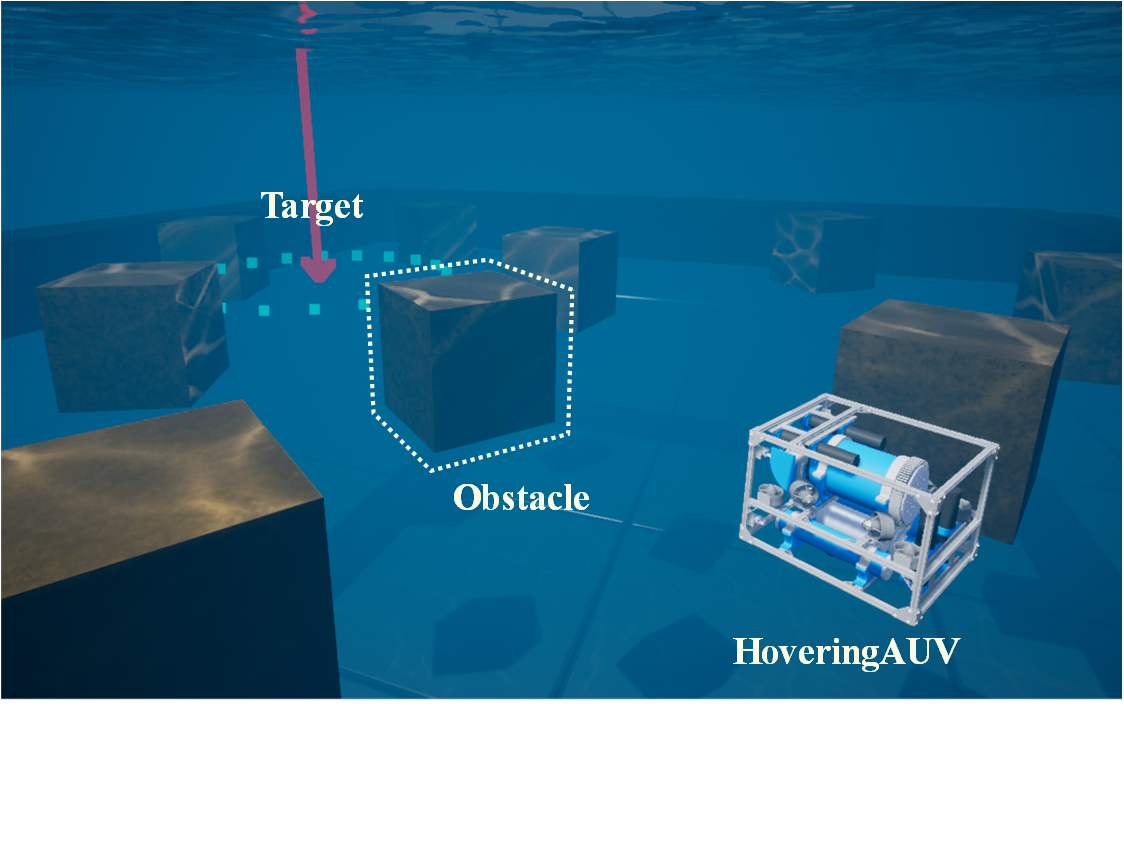}
\label{fig:env_holoocean}
	}
   \hspace{-9mm}
    \subfigure[]{
\includegraphics[scale=0.22]{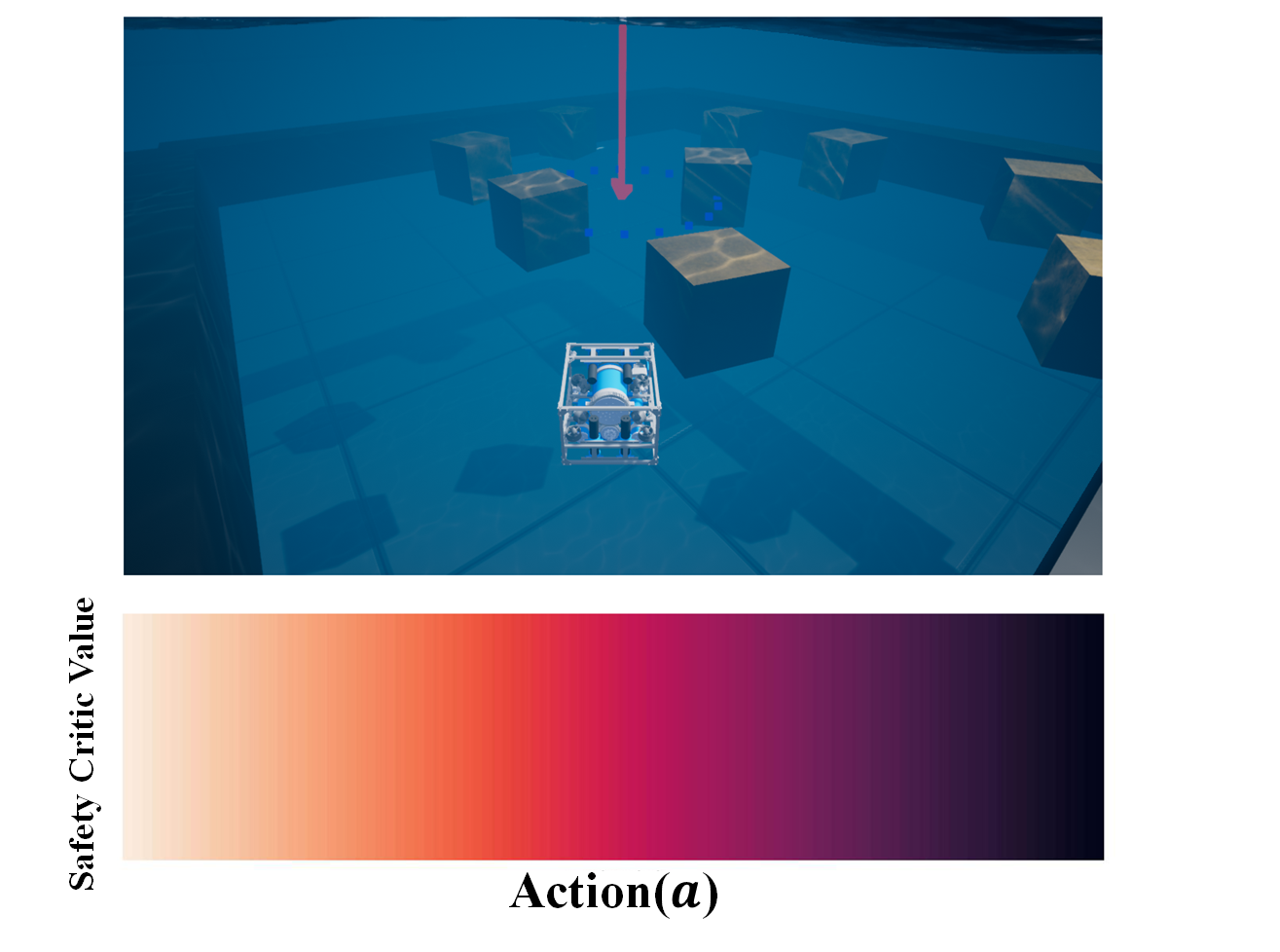}
	}
   \hspace{-13mm}
    \subfigure[]{
\includegraphics[scale=0.22]{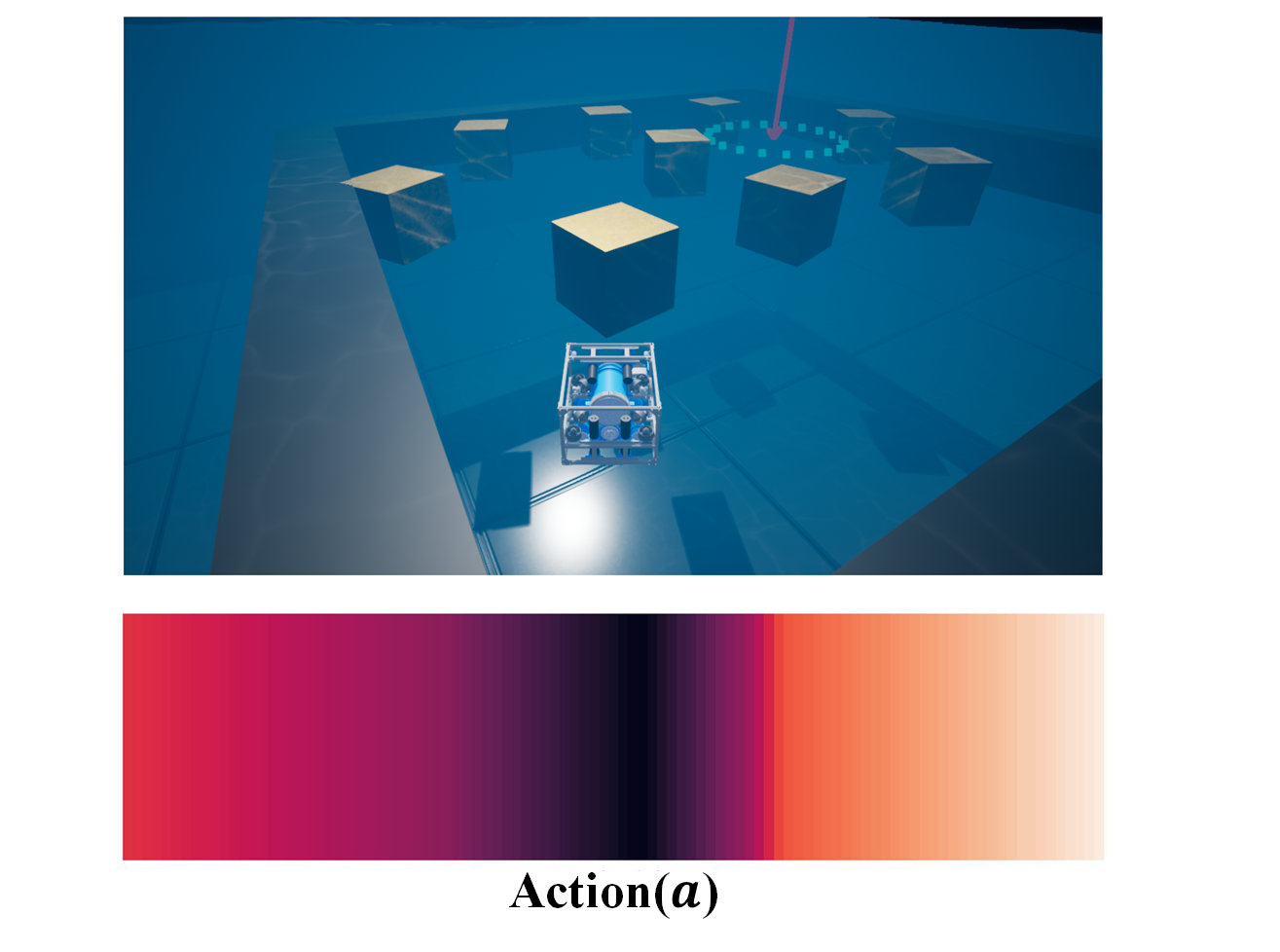}
	}
    \hspace{-13mm}
     \subfigure[]{
\includegraphics[scale=0.22]{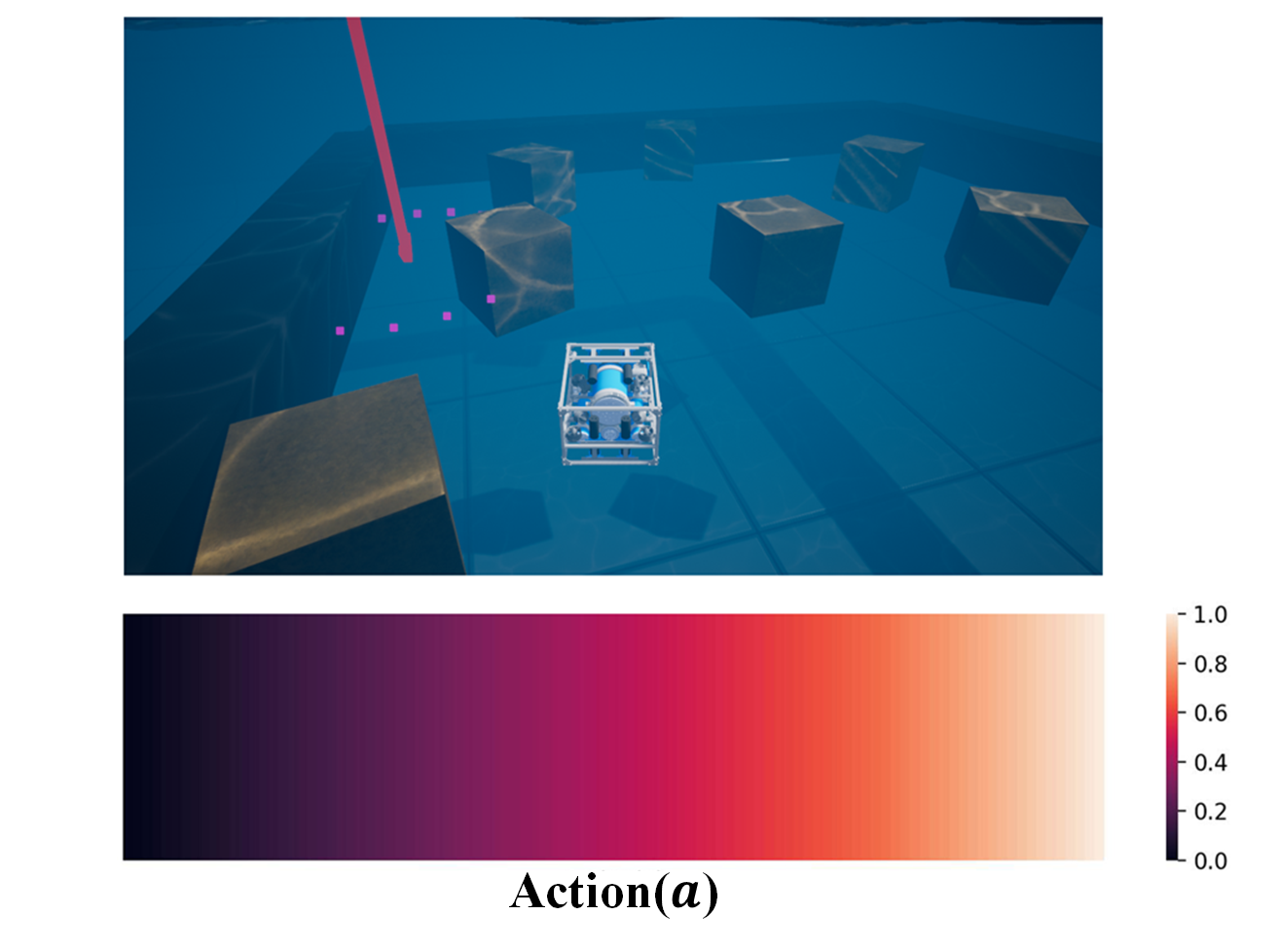}
	}
 \caption{(a): The simulation environment in HoloOcean. (b-d): Results of safety critic Q-values under particular observations. The horizontal coordinate of the heat map represents the action from -1 rad/s to 1 rad/s. The color shade represents the safety critic Q-value for performing the corresponding action under the current observation. The values are normalized to be distinguished obviously.}
 \label{fig:turn}
\end{figure*}

The simulation environment is a $10m \times 10m$ pool, as depicted in \autoref{fig:env_holoocean}. The AUV is initialized randomly at one of the four corners of the pool. The obstacles, which are cubic, are randomly generated in the pool and distributed densely. The target position is randomly generated in the other diagonal half area of the pool where the starting point is not located.  When the distance between the AUV and the target position is less than $1.5m$, we consider reaching the target is achieved.

The CBF used in reward function design for safety is $h(s)=\alpha(min\left \{ {d_i}^2 \right \} -d_{thres}^2), i = 1,\dots,9$, where $d_{thres}$ is the threshold distance for safety, set as $1m$ in the experiment, and $\alpha$ is the coefficient to adjust the magnitude of the results, set as 5. Similarly, the CLF used for navigation task is $l(s)={(x-x_{tar})}^2+{(y-y_{tar})}^2$.

\begin{table}[!t]
\caption{Ablations Results in 100 Random Episodes}
\label{tab:performance}
\begin{center}
\vspace{-4mm}

\begin{tabular}{lcccc}
    \toprule
    {Algorithm} &Success Rate$^a$ $\uparrow$&Collision Rate $\downarrow$&Average length$^b$$\downarrow$\\
    \midrule
    T-O(0.5)$^c$  &48\%  &31\%  &852\\
    T-O(0.25) &67\%  &33\%  &679\\
    Stage 1  &N/A &18\%  &N/A\\
    W/o Re.$^d$ &49\%  &51\%  &\textbf{616}\\
    Our CAC & \textbf{86\%}& \textbf{11\%}&625\\
    \bottomrule
    \vspace{-0.25cm}
\end{tabular}
\begin{tablenotes}
\item a): reaching target position without collision is regarded as a successful episode;
\item b): the average length of successful episodes;
\item c): T-O($x$) represents the reward function is set as a trade-off with $x$ as the coefficient of safety reward, and $1-x$ as that of navigation reward;
\item d): policy update only using $\nabla_\theta J_2(\theta)$ without restriction \eqref{eq_restricted-policy-update}.
\end{tablenotes}
\end{center}
\vspace{-0.4cm}
\end{table}

The critic networks are 3 layers fully connected net with 256 units per layer and actor network has 2 layers of the same structures. Both the first and second stages of our CAC algorithm are trained for 5 million steps. The results of our approach are also shown in \autoref{fig:traj}, from which we can find that the CAC algorithm performs well in both collision avoidance and goal-reaching. Ablation studies are further conducted for four models: the DRL models that use just one reward function as a trade-off between safety and navigation rewards in the form $xr_1+(1-x)r_2$; only stage 1 of the CAC algorithm; the CAC variant that updates the policy using $\nabla_\theta J_2(\theta)$ only 
without consideration of restriction \eqref{eq_restricted-policy-update}. As shown in \autoref{tab:performance}, the CAC algorithm achieves the highest rate of reaching the target while maintaining the lowest rate of collision. A further investigation on how the restriction \eqref{eq_restricted-policy-update} affects the learned policy is conducted. As shown in Fig. \ref{fig:ab_curve}, the integration of policy restriction facilitates learning policy with better performance in reaching navigation targets, which is evident by higher navigation rewards and longer episode length during training.

We verify whether the safety critic is correct in the experiment via various scenarios the AUV would face. The results of three representative scenarios are depicted in \autoref{fig:turn}, which show that the actions that avoid collision, shadowed by lighter colours, have higher critic values and are better choices from a safety perspective. It verifies that our safety critic can correctly reflect the level of safety of AUV navigation behaviours. 

\section{Conclusions}
In this paper, we propose a model-free reinforcement learning algorithm called certificated Actor-Critic for safe robot navigation. We present a well-defined reward function generated by CBFs, and prove the relationship between safety and value functions. Furthermore, we design a hierarchical framework and restricted policy update strategy to realize goal-reaching without compromising safety. Finally, the CartPole and AUV experiments are conducted, showing the superior navigation performance of our algorithm in terms of safety and goal reaching. 



\bibliographystyle{IEEEtran}

\end{document}